\documentclass{article}
\usepackage{amsmath, amssymb, amsthm}
\usepackage[utf8]{inputenc}
\usepackage[T1]{fontenc}
\usepackage[english]{babel}
\usepackage[a4paper, margin=1in]{geometry}  
\usepackage[
colorlinks=true,
linkcolor=blue,
citecolor=blue,
urlcolor=blue,
pdfborder={0 0 0},
]{hyperref}
\usepackage{comment}
\usepackage{comment}
\usepackage{multirow}
\usepackage{float} 
\usepackage{placeins} 
\usepackage{cancel}

\usepackage{bm} 
\newcommand{\balpha}{\bm{\alpha}}

\newcommand{\bx}{\mathbf{x}}
\newcommand{\bc}{\mathbf{c}}
\newcommand{\bC}{\mathbf{C}}

\newcommand{\bj}{\mathbf{j}}

\usepackage{graphicx}
\newtheorem{theo}{Theorem}
\newtheorem{lem}{Lemme}
\newtheorem{pro}{Proposition}
\newtheorem{coro}{Corollary}

\newtheorem{rem}{Remark}

\usepackage{orcidlink}  

\usepackage{doi}  
\usepackage[numbers]{natbib}

\usepackage{lineno}

\title{Approximation with SiLU Networks: Constant Depth and Exponential Rates for Basic Operations}

\author{
	Koffi O. AYENA\orcidlink{0009-0009-0858-258X} \\
	\small Université de Lomé, Laboratoire de Modélisations Mathématiques et Applications, Lomé, Togo \\
	\small Université de Belfort de Montbéliard, Laboratoire Interdisciplinaire Carnot de Bourgogne, \\
	\small ICB/UTBM, UMR 6303 CNRS, Belfort, France \\
	\small \texttt{kayena257@gmail.com} \\
	\small \texttt{koffi.ayena@utbm.fr}
}
\date{\today}

\begin{document}
	\maketitle

	\begin{abstract}
		\noindent
		We present SiLU network constructions whose approximation efficiency depends critically on proper hyperparameter tuning. For the square function $x^2$, with optimally chosen shift $a$ and scale $\beta$, we achieve approximation error $\varepsilon$ using a two-layer network of constant width, where weights scale as $\beta^{\pm k}$ with $k = \mathcal{O}(\ln(1/\varepsilon))$. We then extend this approach through functional composition to Sobolev spaces, we obtain networks with depth $\mathcal{O}(1)$ and $\mathcal{O}(\varepsilon^{-d/n})$ parameters under optimal hyperparameters settings. Our work highlights the trade-off between architectural depth and activation parameter optimization in neural network approximation theory.
	\end{abstract}
	


	\section{Introduction}
	\label{sec1}
	The study of approximation capabilities of neural networks is a central topic in the theory of deep learning. Since the classical results of Cybenko \cite{cybenko1989approximation}, Hornik \cite{hornik1991approximation}, and Funahashi \cite{funahashi1989approximate}, it is known that single-hidden-layer neural networks (NN) with nonlinear activation functions can approximate any continuous function on a compact domain to arbitrary accuracy. These universal approximation theorems, however, do not address efficiency, i.e., how network size relates to approximation error for specific function classes.
	
	Barron \cite{cheang2001penalized} introduced risk bounds for function estimation with single-hidden-layer networks, using a penalized least-squares criterion for model selection. The development of deep networks and the ReLU activation function later led to sharper approximation results. Yarotsky \cite{yarotsky2017error} established exponential expressivity of ReLU networks for smooth functions, and Petersen and Voigtlaender \cite{petersen2018optimal} derived optimal bounds for function classes with variable regularity. The roles of network depth and width in approximation efficiency were clarified in \cite{telgarsky2016benefits, lu2017expressive, arora2016understanding}.
	
	In addition to ReLU, other perspectives on neural network expressivity have been developed. Daniely, in \cite{daniely2020neural}, analyzed learning mechanisms beyond over-parameterization. More recently, smooth activation functions have been reconsidered, after initially being set aside, due to their empirical performance. The SiLU (or swish) activation \cite{elfwing2018sigmoid, ramachandran2017searching} and its variant GELU \cite{hendrycks2016gaussian, lee2023gelu} have been adopted in modern architectures for variational surrogate model \cite{ayena:hal-05178445}.
	
	On the theoretical side, several recent works have examined the approximation and learning properties of smooth activations. Yet, a systematic characterization of the approximation power of SiLU networks, comparable to what is known for ReLU, has not been established.
	
	The purpose of this work is to address this gap by providing a rigorous analysis of the approximation properties of SiLU-activated networks, including convergence rates and complexity bounds, thereby situating them within the general framework of neural network approximation theory.
	
	The remainder of this article is organized as follows. Section~\ref{sec2} provides essential background on neural network definitions and uniform approximation concepts, establishing the foundational framework for our analysis. In Section~\ref{sec3}, we present our main contributions, beginning with the fundamental approximation of the square function using SiLU networks, which serves as the building block for subsequent results.
	This section contains several key contributions:
	
	\begin{itemize}
		\item With weights scale as $ \mathcal{O}(\ln(1/\varepsilon))$, Theorem~\ref{theo1} demonstrates that the square function $x^2$ can be approximated with exponential accuracy $\mathcal{O}(\omega^{-2k})$ by a shallow SiLU network with only 3 neurons and constant depth. Theorem~\ref{theo2} extends this to multiplication $x \cdot y$ , requiring just 4 neurons in a single hidden layer.
		
		\item Theorem~\ref{theo3} generalizes to arbitrary monomials $x^m$, showing they can be approximated with networks of depth $\mathcal{O}(m)$ and constant width.
		
		\item Along similar lines, Theorem~\ref{theo4} presents an alternative single-hidden-layer construction using finite differences, achieving monomial approximation with $m+1$ neurons.
		
		\item Theorem~\ref{theo5} combines these results to approximate arbitrary polynomial functions.

		\item Corollary~\ref{coro1} reveals that the monomial approximation networks naturally form a recurrent neural network (RNN) with shared parameters, efficiently generating polynomial bases $\{1, x, x^2, \ldots, x^m\}$ through sequential computation.

		\item For Continuous function approximation, Theorem~\ref{theo6}, Proposition~\ref{pro1} and~\ref{pro2} and Corollary~\ref{coro2} develop a step-function approach, with constant depth and size $N = \mathcal{O}\left( \frac{b-a}{\omega_f^{-1}(\varepsilon)} \right)$ depending on the modulus of continuity $\omega_f$.
		
		\item Theorem~\ref{theo8} establishes optimal approximation rates for functions in Sobolev spaces $\mathcal{F}_{n,d}$, showing that SiLU networks achieve error $\varepsilon$ with size $\mathcal{O}(\varepsilon^{-d/n})$ and constant depth, eliminating the logarithmic factor present in comparable ReLU constructions.
	\end{itemize}

	\noindent Section~\ref{sec4} discusses practical implications and presents experimental validation of our theoretical results. We analyze the optimization of key parameters ($\beta$, $k$, and $a$) in our constructions and provide comparative experiments between SiLU and ReLU networks. 
	
	Our work contributes to the growing literature on the expressive power of neural networks \cite{petersen2018optimal,lu2017expressive, telgarsky2016benefits} by providing explicit construction methods with guaranteed error bounds. The results have implications for understanding why deep networks with smooth activations perform so effectively in practice and provide theoretical guidance for the design of network architectures. The constructive nature of our proofs offers potential pathways for developing specialized networks for scientific computing applications where high precision is required.

	\section{Definition of neural network and uniform approximation}
	\label{sec2}
	\paragraph{Definition}
	
	A neural network of depth $m$ (i.e., $m$ layers including input and output, or $(m-2)$ hidden layers) is a function $N : \mathbb{R}^p \to \mathbb{R}^d$ defined as the composition of affine maps and nonlinear activation functions.
	More precisely, the network is given by
	$$
	N(\bx) = (\phi_{m} \circ h_m) \circ (\phi_{m-1} \circ h_{m-1}) \circ \cdots \circ (\phi_1 \circ h_1)(\bx),
	$$
	where for each $i \in \{1, \dots, m\}$,
	\begin{itemize}
		\item  the $i$-th hidden layer contains $n_i$ neurons, corresponding to the $n_{i}$ coordinates of $\phi_i\circ h_i(\bx_i)$ computation is given by an affine transformation
		$h_i : \mathbb{R}^{n_{i-1}} \to \mathbb{R}^{n_i}$ of the form
		$$
		h_i(\bx_i) = W_i \bx_i + b_i,
		$$
		with the input vector $ \ \bx_i \in \mathbb{R}^{n_{i-1}}$, weight matrix $W_i \in \mathbb{R}^{n_i \times n_{i-1}},$ and bias vector $b_i \in \mathbb{R}^{n_i}$
		and  $\phi_i : \mathbb{R}^{n_i} \to \mathbb{R}^{n_i}$ is a nonlinear activation function applied componentwise (typically $\phi_m = \text{id}_{\mathbb{R}^{n_m}}$ for the output layer),
		\item and the dimensions satisfy $n_0 = p$ and $n_m = d$.
	\end{itemize}
	
	Let $f:[-B,B]^p\to \mathbb{R}$ be a function. We say that a function $g$ approximates $f$ with precision $\varepsilon>0$ if
	
	$$
	\sup_{\mathbf{x} \in [-B, B]^d} |g(\mathbf{x}) - f(\mathbf{x})| \leq \varepsilon.
	$$

	\noindent\textbf{Convention.} If $\bx = (x_1,\dots,x_p) \in \mathbb{R}^p$, then
	\[
	\text{SiLU}(\bx) := \big(\text{SiLU}(x_1), \dots, \text{SiLU}(x_p)\big) \in \mathbb{R}^p,
	\]
	
	that is, SiLU is applied componentwise to vectors, where $$\forall \ x \in  \mathbb{R}, \quad \text{SiLU}(x) = \frac{x}{1+e^{-x}}.$$ 
	The notation \(h(\varepsilon)=\underset{\varepsilon\to 0}{\mathcal{O}}(1)\) means that $h$ does not depend on the $\varepsilon$, it remains bounded as $\varepsilon\to 0$.
	In the following, we simply write $\mathcal{O}(1)$ with $\varepsilon\to 0$ being clear from context.
	\begin{figure}[htb]
		\centering
		\includegraphics[width=\linewidth, height=8cm, keepaspectratio]{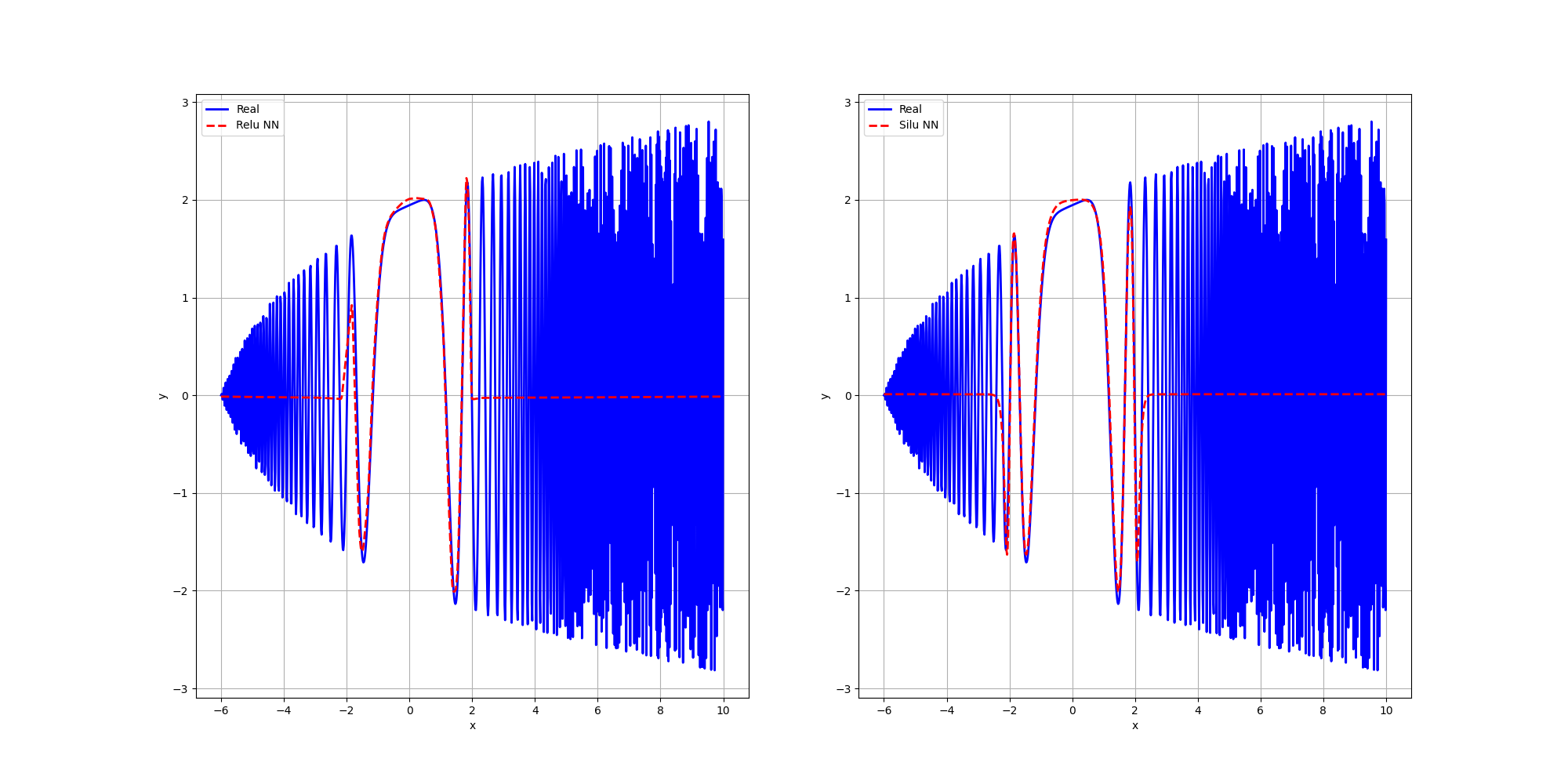}
		\caption{Approximation of the function $\log(7+x)\cos(x^3)$ by ReLU neural networks (left, Mse =1.75) and SiLU neural networks (right, MSE=1.69).}
		\label{fig1}
	\end{figure}
	
	While deep ReLU networks achieve optimal worst-case approximation rates for Sobolev spaces \cite{yarotsky2017error}, empirical evidence suggests that smooth activations like SiLU can sometimes achieve lower error on specific regular functions. For instance, in approximating \(\log(7+x)\cdot\cos(x^3)\) with identical architectures, we observe significantly lower error using SiLU compared to ReLU (Fig.~\ref{fig1}). This observation motivates our investigation into whether smooth activations might offer improved constants or more favorable depth–width trade-offs for functions with particular smoothness or algebraic structure, complementing ReLU's asymptotic optimality for general function classes.
	
	To systematically investigate these potential advantages, we develop a theoretical framework for analyzing SiLU network approximation. Our analysis operates within an extended model where activation parameters can be precision-dependent—acknowledging that practical implementations often tune such parameters. Within this framework, we derive explicit complexity bounds and compare them with established ReLU constructions. This allows us to precisely characterize trade-offs between architectural depth, parameter count, and numerical precision requirements.
	
	\section{Theoretical results: Efficient polynomial approximation with SiLU}
	\label{sec3}
	\subsection{A constant-depth building block: Approximating the square function}
	
	In this section, we establish the theoretical foundations for this efficiency. We begin by showing how SiLU networks approximate the most basic nonlinear monomial, the square function \(x^2\), with constant depth and exponential accuracy, where weights scale as $\beta^{\pm k}$ with $k = \mathcal{O}(\ln(1/\varepsilon))$ (Fig~\ref{fig2}).
	
	Our construction will serve as the fundamental building block for all subsequent approximations.

	\begin{theo}
		\label{theo1}
		For any $ B > 0 $ and target accuracy $\varepsilon \in (0,1)$ , there exist constants $C_1 > 0$, $\omega_1 > 1$ and a family of feedforward SiLU neural networks $\{Q_k\}_{k \in \mathbb{N}}$, with one hidden layer with $3$ neurons each, such that for $k \geq  \left\lceil\frac{\ln C_1 - \ln \varepsilon}{2\ln \omega_1}\right\rceil$
		\[
		\sup_{|x| \leq B} |Q_k(x) - x^2| \leq \varepsilon,
		\]
		where \( \lceil \cdot \rceil \) is the ceiling function. 
		Moreover, the network $Q_k$ has depth $\mathcal{O}(1)$ and size $\mathcal{O}(1)$.
	\end{theo}
	\begin{proof}
		Let $\sigma(x) = (1+e^{-x})^{-1}$ and $\text{SiLU}(x)=x\sigma(x)$.
		Define $g_a(x) = \text{SiLU}(x+a) + \text{SiLU}(-x+a) - 2\text{SiLU}(a)$.
		
		Expanding $\sigma(z)$ around $z=a$:
		\begin{align*}
			\sigma(z) &= \sigma(a) + \sigma'(a)(z-a) + \frac{\sigma''(a)}{2}(z-a)^2 + \cdots
		\end{align*}
		Then, using $s=\sigma(a)$, $s'=\sigma'(a)$, $s''=\sigma''(a)$, and 
		\begin{align*}
			(a \pm x)\sigma(a \pm x) &= as \pm (s+as')x + \Big(s' + \frac{a s''}{2}\Big)x^2 \pm \cdots +,
		\end{align*}
		We have 
		\[
		g_a(x) = (2s' + a s'')x^2 + \sum_{j=2}^\infty c_{2j}(a)x^{2j},
		\]
		where $c_{2j}(a) = \frac{2\bigl(s^{(2j-1)}/(2j-1)! + a s^{(2j)}/(2j)!\bigr)}{2s' + a s''}$.
		
		\noindent	Assuming $K(a)=2s'+as'' \neq 0$, set 
		\[
		f_a(x)=g_a(x)/K(a)=x^2 + \sum_{j=2}^\infty c_{2j}(a)x^{2j}.
		\]
		Define iteratively for $k\ge 0$ with $\beta\in(0,1)$:
		\[
		Q_0(x)=f_a(x),\qquad Q_k(x)=\beta^{-2}Q_{k-1}(\beta x)=\beta^{-2k}f_a(\beta^k x).
		\]
		Hence
		\[
		Q_k(x)=x^2 + \sum_{j=2}^\infty c_{2j}(a)\beta^{k(2j-2)}x^{2j}.
		\]
		\ref{bound} gives 
		\[
		\Bigl|\frac{\sigma^{(m)}(a)}{m!}\Bigr| \le C_a\, \rho^{\,m},
		\qquad \rho=\frac{2}{\pi},\; C_a>0.
		\]
		Consequently,
		\[
		|c_{2j}(a)| \le \frac{2C_a(1+|a|\rho)}{|K(a)|}\,\rho^{\,2j-1}
		=: \widetilde{C}_a\,\rho^{\,2j}.
		\]
		For $|x|\le B$,
		\begin{align*}
			|Q_k(x)-x^2|
			&\le \sum_{j\ge 2} |c_{2j}(a)|\,\beta^{k(2j-2)}B^{2j} \\
			&\le \widetilde{C}_a \rho^4 B^{4}\beta^{2k}\sum_{j\ge 2} (\rho^2B^{2}\beta^{2k})^{(j-2)} \\
			&= \widetilde{C}_a\,\rho^4 B^4\,\frac{\beta^{2k}}{1-\rho^2B^2\beta^{2k}} \quad\text{if }\rho^2B^2\beta^{2k}<1.
		\end{align*}
		Choosing $\beta$ small enough that $\rho^2B^2\beta^2\le\frac12$, we obtain for all $k\ge1$
		\[
		|Q_k(x)-x^2| \le 2\widetilde{C}_a\,\rho^4B^4 \beta^{2k}=C_1\beta^{2k},
		\]
		with $\beta^{-1}=\omega_1$.
		
		Choosing $k \geq \bigl\lceil\frac{\ln C_1 -\ln \varepsilon}{2\ln \omega_1}\bigr\rceil= \mathcal{O}\Bigl(\ln 1/\varepsilon\Bigr)$ ensures $|Q_k(x)-x^2|\le\varepsilon$ for all $|x|\le B$. 
		Indeed,
		\begin{figure}[htb]
			\centering
			\includegraphics[width=\linewidth, height=8cm, keepaspectratio]{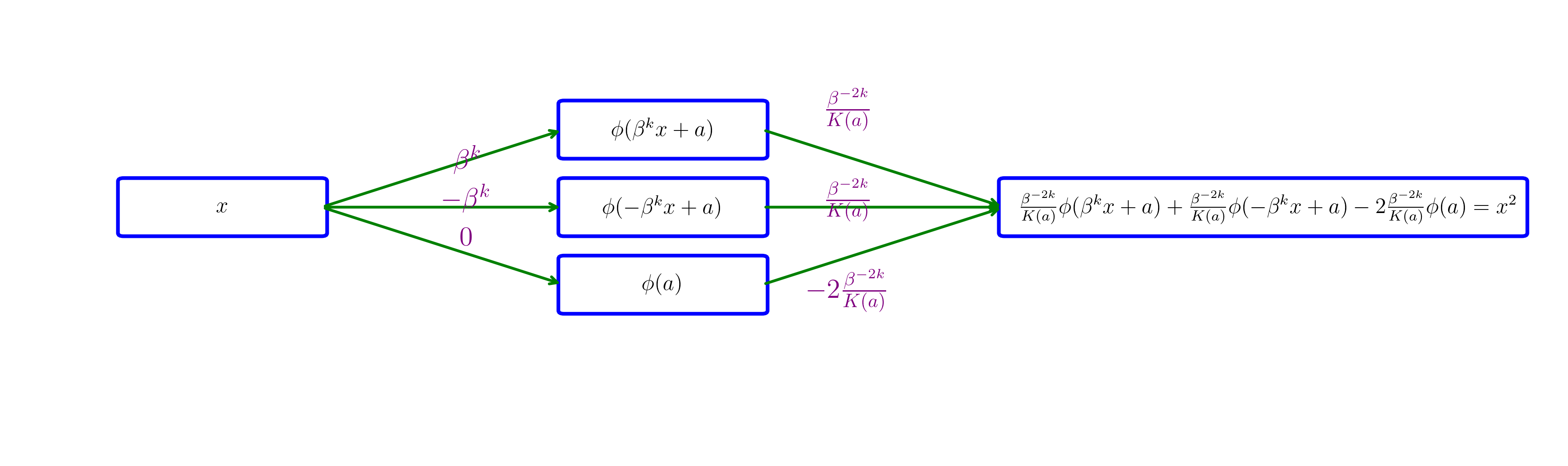}
			\caption{Architecture of NN $Q_k$ (Theorem \ref{theo1}).}
			\label{fig2}
		\end{figure}
		\begin{equation}
			\label{Architecture_reseauQ}
			Q_k(x)=\gamma\cdot \text{SiLU}(Wx+b)
		\end{equation}
		with parameters such that the weights $W=[\beta^{k},-\beta^{k}, 0]^\top$, $\gamma=\frac{\beta^{-2k}}{K(a)}[1,1,-2]$, and the bias $b=[a,a,a]^\top$.
	\end{proof}
	\noindent The theoretical convergence shown above is confirmed numerically in (Fig.~\ref{fig3}), which displays the approximation of the function \(x^2\) by the NN \(Q_k\) with \(a=0\) and \(\beta=0.27\)(Fig.~\ref{fig2}).

	\begin{figure}[htb]
		\centering
		\includegraphics[width=\linewidth, height=6cm, keepaspectratio]{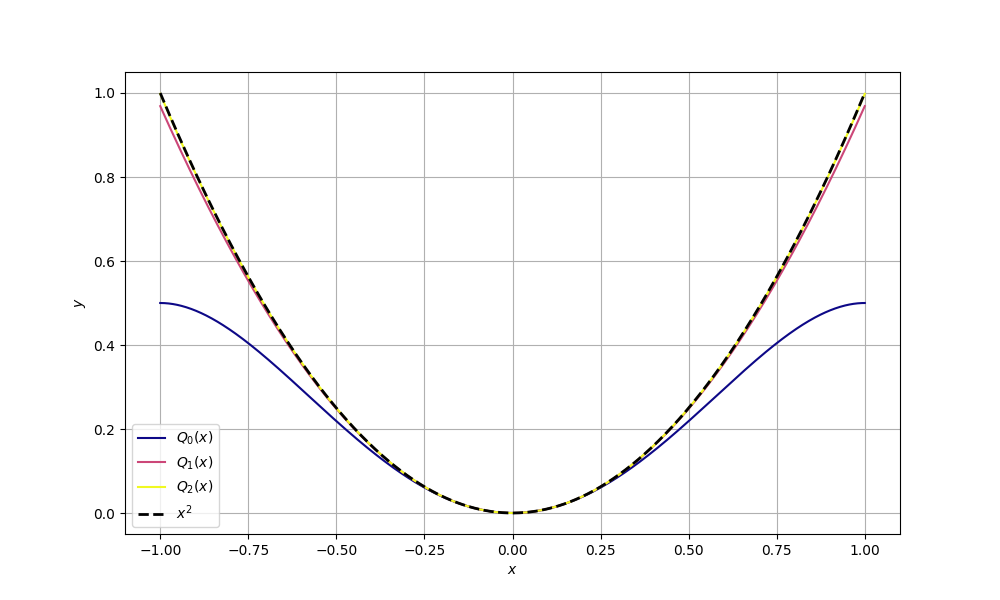}
		\caption{Approximation of the function $x^2$ by the NN $Q_k$ with $a=0$ and $\beta=0.27$.  }
		\label{fig3}
	\end{figure}
	
	\subsection{Efficient Multiplication via Squaring}
	We use the classical algebraic identity 
	
	\begin{equation}
		\label{eq}
		xy = \frac{1}{4}[(x+y)^2 - (x-y)^2]
	\end{equation}
	to establish the ability of neural networks with SiLU activation functions to efficiently approximate the multiplication of two variables. 
	\begin{theo}
		\label{theo2}
		For every $B > 0$ and target accuracy $\varepsilon \in (0,1)$, there exist constants $C_2 > 0$, $\omega_2 > 1$ and a family of feedforward SiLU neural networks $\{M_k\}_{k \in \mathbb{N}}$, with at most one hidden layer with $4$ neurons each, such that for $k \geq \left\lceil\frac{\ln C_2 - \ln \varepsilon}{2\ln \omega_2}\right\rceil$  
		
		$$\sup_{|x|,|y| \leq B} |M_k(x, y) - xy| \leq \varepsilon $$
		
		The network $M_k$ has depth $\mathcal{O}(1)$ and size $\mathcal{O}(1)$.
	\end{theo}
	\begin{proof}
		Using the identity \(xy = \frac14[(x+y)^2 - (x-y)^2]\), we construct an approximant \(M_k(x,y)\) from the square-approximating network \(Q_k\) of Theorem~\ref{theo1}.
		For \(|x|,|y|\le B\), we have \(|x\pm y|\le 2B\), so Theorem~\ref{theo1} guarantees
		\[
		\sup_{|x|,|y|\le B} |Q_k(x\pm y)-(x\pm y)^2| \le C_1\omega_1^{-2k}
		\]
		with constants \(C_1>0,\ \omega_1>1\). Hence
		\begin{align*}
			\sup_{|x|,|y|\le B} |M_k(x,y)-xy| 
			&\le \frac14\bigl( |Q_k(x+y)-(x+y)^2| + |Q_k(x-y)-(x-y)^2| \bigr) \\
			&\le C_2\omega_1^{-2k},
		\end{align*}
		where \(C_2 = C_1/2\) and $\omega_2=\omega_1$.
		
		Choosing \(k \geq \left\lceil\frac{\ln C_2 - \ln \varepsilon}{2\ln \omega_2}\right\rceil\) yields an error \(\le\varepsilon\). Since \(Q_k\) has constant depth and size, and the linear operation adds only \(\mathcal{O}(1)\) layer, the total network \(M_k\) also has constant depth and size.
		
		Explicitly, \(M_k\) can be written as a SiLU network with weight matrix
		\[
		W = \beta^{k}\begin{bmatrix}
			+1&+1\\
			-1&-1\\
			+1&-1\\
			-1&+1
		\end{bmatrix},\quad
		\gamma = \frac{\beta^{-2k}}{4K(a)}\begin{bmatrix}1&1&-1&-1\end{bmatrix}^\top,
		\]
		bias \(b = (a,a,a,a)^\top\), and input \(z=(x,y)^\top\).
	\end{proof}
	\noindent Experimentally, we employ the neural network \(\{M_k\}\) (whose architecture is shown in (Fig.~\ref{fig4}) with parameters \(a=0\) and \(\beta=0.27\) to approximate the product \(\cos(x) \cdot \sin(x)\log(x^2)\) via \(M_k\bigl(\cos(x), \sin(x)\log(x^2)\bigr)\). The results are presented in Fig.~\ref{fig5}.
	
	\begin{figure}[htb]
		\centering
		\includegraphics[width=\linewidth, height=6cm, keepaspectratio]{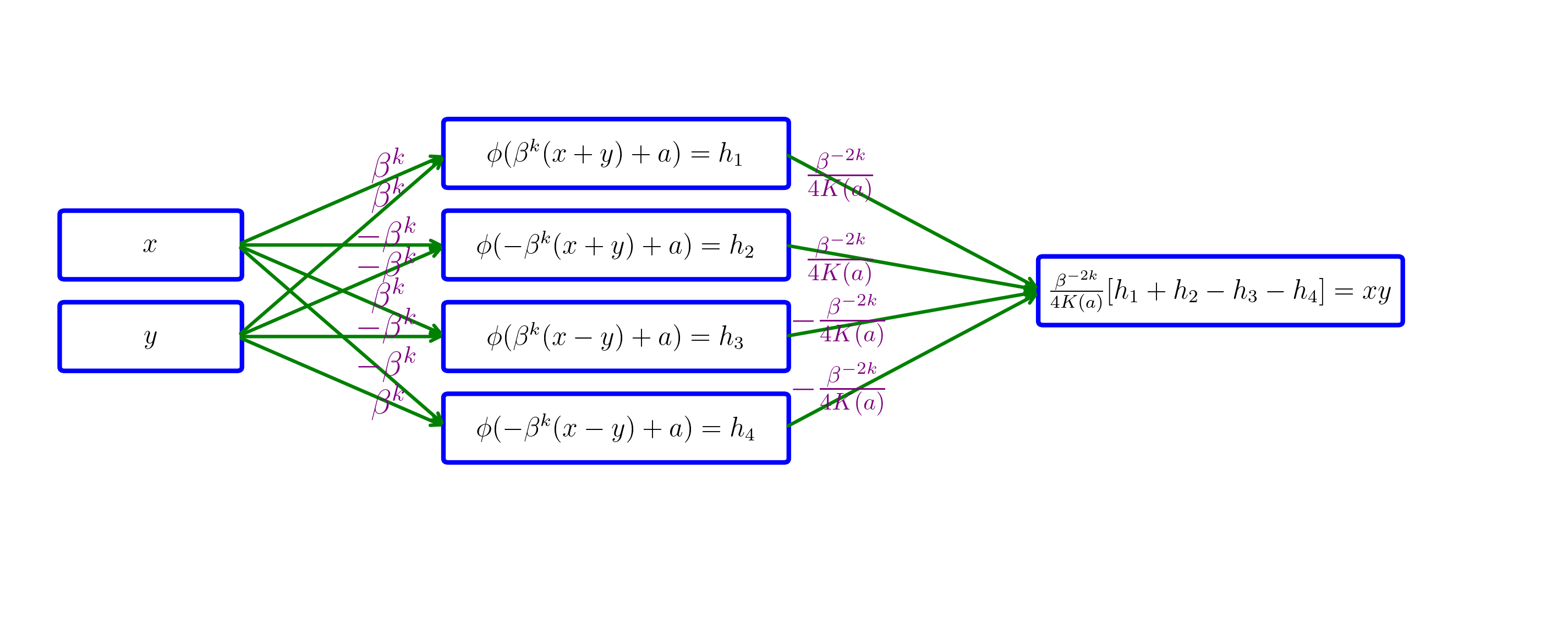}
		\caption{Architecture of NN $M_k$ (Theorem \ref{theo2}).}
		\label{fig4}
	\end{figure}
	\begin{figure}[htb]
		\centering
		\includegraphics[width=\linewidth, height=6cm, keepaspectratio]{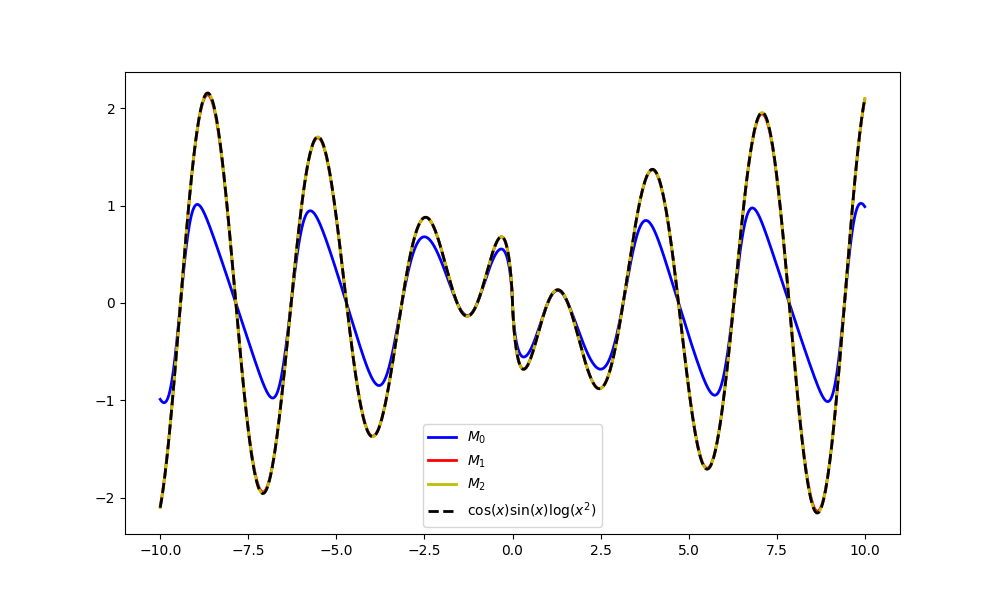}
		\caption{Approximation $M_k\left(\cos(x), \sin(x)\log(x^2)\right)$ of the product of the functions $\cos(x)$ and $\sin(x)\log(x^2)$ by the NN $\{M_k\}$ with $a=0$ and $\beta=0.27$.}
		\label{fig5}
	\end{figure}
	
	\subsection{Approximation of Monomials}
	
	\subsubsection{Deep Network Construction with Constant Width}
	This next result generalizes the previous approximation theorem from the product of two variables to arbitrary monomials $x^m$. It shows that SiLU networks can approximate any power function xmxm with exponential accuracy, using a network whose depth grows linearly with the degree $m$
	\begin{theo}
		\label{theo3}
		For every $B > 0, m\geq 2,$ and target accuracy $\varepsilon \in (0,1)$, there exist constants $C_{m} > 0$, $\omega_{m} > 1$ and a family of feedforward SiLU neural networks $\{P_{m,k}\}_{k \in \mathbb{N}}$, with at most $2m$ hidden layers with $4$ neurons each, such that for $k \geq \left\lceil\frac{\ln C_m - \ln \varepsilon}{2\ln \omega_m}\right\rceil$ 
		
		\[
		\sup_{|x| \leq B}|P_{m,k}(x) - x^m| \leq \varepsilon
		\]
		The network $P_{m,k}$ has depth $\mathcal{O}(m)$ and size $\mathcal{O}(1)$.
	\end{theo}
	\begin{proof}
		By induction on \(m\). For \(m=2\) use Theorem~\ref{theo1} directly.
		
		Assume \(P_{m,k}\) satisfies \(|P_{m,k}(x)-x^m|\le C_m\omega_m^{-2k}\) for \(|x|\le B\) with \(\mathcal{O}(m)\) depth, \(\mathcal{O}(1)\) size. Define
		\[
		P_{m+1,k}(x)=M_k(P_{m,k}(x),x),
		\]
		where \(M_k\) is the product network from Theorem~\ref{theo2}. Since both inputs are bounded by \(B'=B+\max_{|x|\le B}|P_{m,k}(x)|\), we have
		\begin{align*}
			&|P_{m+1,k}(x)-x^{m+1}|\\
			&\le |M_k(P_{m,k}(x),x)-P_{m,k}(x)x| + |P_{m,k}(x)-x^m||x|\\
			&\le (C_2 + B'C_m)\omega_{m+1}^{-2k},
		\end{align*}
		with \(\omega_{m+1} =\min\{\omega_2,\omega_m\}\). Depth increases by \(\mathcal{O}(1)\), size remains \(\mathcal{O}(1)\). The choice 
		$$ k\geq \left\lceil\frac{\ln C_{m+1} - \ln \varepsilon}{2\ln \omega_{m+1}}\right\rceil$$ 
		yields error \(\varepsilon\).
	\end{proof}
	\noindent In practice (Fig.~\ref{fig6}), we achieve an approximation of the function \(x^7\) using the NN \(P_{m,k}\) with parameters \(a=0\) and \(\beta=0.27\).
	
	\begin{figure}[h!]
		\centering
		\includegraphics[width=\linewidth, height=6cm, keepaspectratio]{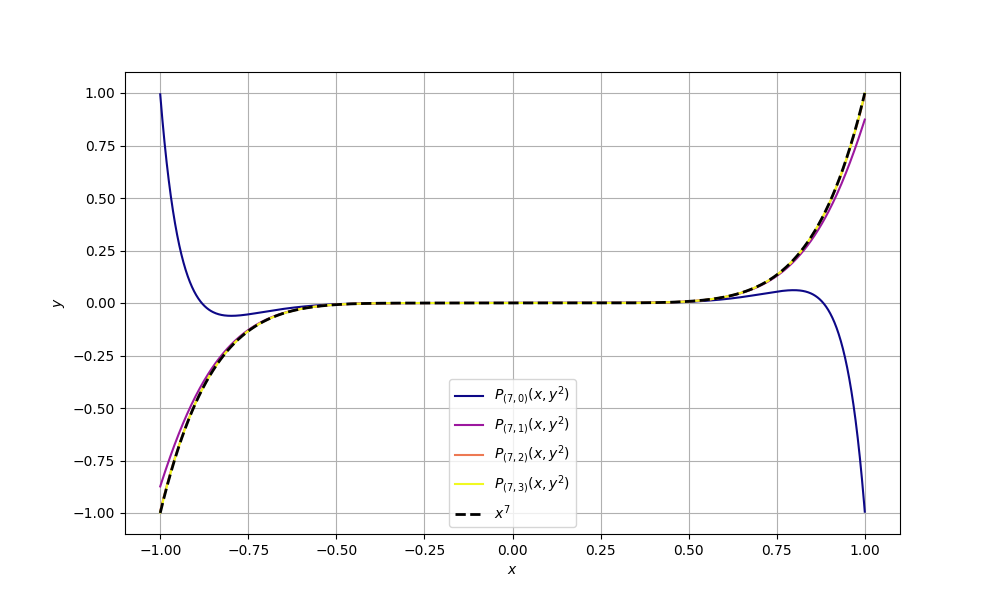}
		\caption{Approximation of the function \(x^7\) by the functions \(P_{m,k}\) with $a=0$ and $\beta=0.27$ (Theorem \ref{theo3}).}
		\label{fig6}
	\end{figure}
	
	\subsubsection{Shallow Construction for Arbitrary Monomials}
	
	While Theorem~\ref{theo3} approximates $x^m$ using a deep network with $\mathcal{O}(m)$ layers, we now present a more efficient shallow construction requiring only a single hidden layer. This architectural improvement is made possible by exploiting the finite difference properties of the SiLU activation function, which allows the polynomial approximation to be compressed into a single nonlinear transformation while maintaining exponential convergence. The shift from depth to moderate width offers practical advantages where depth-related issues like vanishing gradients are a concern.
	
	\begin{theo}
		\label{theo4}
		For any $B > 0$ and target accuracy $\varepsilon \in (0,1)$ , there exist constants $C_{1,m} > 0$, $\omega > 1$ and a family of feedforward SiLU neural networks $\{Q^m_k\}_{k \in \mathbb{N}}$, with one hidden layer with $m+1$ neurons each, such that for $k = \left\lceil  \frac{\ln C_{1,m} -ln\varepsilon}{\ln \omega}\right\rceil$
		\[
		\sup_{|x| \leq B} |Q^m_k(x) - x^m| \leq \varepsilon,
		\]
		
		Moreover, the network $Q^m_k$ has depth $\mathcal{O}(1)$ and size $\mathcal{O}(m)$.
	\end{theo}
	\begin{proof}
		For \(m\ge 1\) and \(a\in\mathbb{R}\), define
		\[
		g_{a,m}(x)=\sum_{j=0}^{m}(-1)^{m-j}\binom{m}{j}
		\mathrm{SiLU}\!\bigl((j-\tfrac{m}{2})x+a\bigr),
		\]
		where \(\mathrm{SiLU}(t)=t\sigma(t)\) and \(\sigma(t)=(1+e^{-t})^{-1}\).
		
		Expanding \(\sigma\) in Taylor series around \(a\) gives
		\begin{align*}
			\mathrm{SiLU}\!\bigl((j-\tfrac{m}{2})x+a\bigr)
			&= a\sigma(a)+\sum_{n\ge1}\frac{(j-\tfrac{m}{2})^n}{n!}
			\bigl[a\sigma^{(n)}(a)+n\sigma^{(n-1)}(a)\bigr]x^n.
		\end{align*}
		Summing over \(j\) with alternating binomial coefficients cancels all terms with \(n<m\) (the discrete \(m\)-th difference of a polynomial of degree \(<m\) vanishes). Hence
		\[
		g_{a,m}(x)=K_m(a)x^m+\sum_{n>m}A_n x^n,
		\]
		where \(K_m(a)=a\sigma^{(m)}(a)+m\sigma^{(m-1)}(a)\neq0\) and
		\(A_n=\frac{\Delta^m p_n(0)}{n!K_m(a)}\bigl[a\sigma^{(n)}(a)+n\sigma^{(n-1)}(a)\bigr]\)
		with \(p_n(t)=(t-\tfrac{m}{2})^n\).
		
		Set \(f_{a,m}(x)=g_{a,m}(x)/K_m(a)=x^m+\sum_{n>m}A_n x^n\). For \(\beta\in(0,1)\) define
		\[
		Q^m_k(x)=\beta^{-mk}f_{a,m}(\beta^k x)=x^m+\sum_{n>m}A_n\beta^{k(n-m)}x^n.
		\]
		Using the bound \(|\Delta^m p_n(0)|\le 2^m(m/2)^n\) (\ref{diff}) and
		\(|a\sigma^{(n)}(a)+n\sigma^{(n-1)}(a)|\le C_a(2/\pi)^n n!\) (\ref{bound}), we obtain
		\(|A_n|\le C_1 (m/\pi)^n\) with \(C_1=2^mC_a/|K_m(a)|\). 
		
		Hence for \(|x|\le B\),
		\begin{align*}
			|Q^m_k(x)-x^m| & \le C_1\beta^{-mk}\sum_{n>m}\bigl(\tfrac{mB}{\pi}\beta^k\bigr)^{\!n}
			\\
			& \le C_1\beta^{-mk}\bigl(\tfrac{mB}{\pi}\beta^k\bigr)^{m+1}\sum_{n\geq 0}\bigl(\tfrac{mB}{\pi}\beta^k\bigr)^{\!n}\\
			& \le C_1\beta^{k}\bigl(\tfrac{mB}{\pi}\bigr)^{m+1}\frac{1}{1 - \frac{mB \beta^k}{\pi}}\\
			& \le C_{1,m}\,\beta^k,
		\end{align*}
		where \(C_{1,m}=2 C_1 \left(\frac{mB}{\pi}\right)^{m+1}\) provided \(\beta^k mB/\pi<1/2\). 
		Thus
		$$	\sup_{|x| \leq B} |Q^m_k(x)-x^m |\le C_{1,m}\omega^{-k},$$
		with \(\omega=\beta^{-1}\).
		Choosing \(k=\mathcal{O}(\ln 1/\varepsilon)\) yields error \(\varepsilon\). The network \(Q^m_k\) is a single hidden layer SiLU network with \(m+1\) neurons:
		\[
		Q^m_k(x)=\gamma_m\cdot\mathrm{SiLU}(W_m x+b_m),
		\]
		where \(W_m=\beta^k[w_0,\dots,w_m]^\top\) with \(w_j=j-\tfrac{m}{2}\),
		\(\gamma_m=\beta^{-mk}K_m(a)^{-1}[c_0,\dots,c_m]\) with \(c_j=(-1)^{m-j}\binom{m}{j}\),
		and \(b_m=a\mathbf{1}_{m+1}\),
		where \( \mathbf{1}_{m+1} = [1, 1, \dots, 1]^\top \in \mathbb{R}^{m+1} \) is the all-ones vector of dimension \( m+1 \).
	\end{proof}

	\begin{rem}
		For $a=0$, we have $K_m(0)=m f^{(m-1)}(0)$ where $f(x)=\sigma(x)-1/2$ is odd. Consequently, if $m$ is odd and $m \geq 3$, then $f^{(m-1)}(0)=0$ and thus $K_m(0)=0$.
	\end{rem}

	\subsubsection{A Natural Recurrent Neural Network (RNN) Construction}
	
	A crucial observation emerges from our construction: approximation network parameters are naturally shareable.
	Consequently,
	\begin{coro}
		\label{coro1}
		The resulting network \(P_{m,k}\) of Theorem~\eqref{theo3} is a shallow Recurrent Neural Network (RNN) with at most $2m$ hidden layer of $4$ neurons per step $i$ whose parameters are fixed.
	\end{coro}
	\begin{proof}
		The parameters remain unchanged $W_i\equiv W,\; b_i\equiv b,\; \gamma_i\equiv\gamma$ for all $i$, making the architecture recurrent: it is exactly a Recurrent Neural Network (RNN) where the same operator is applied $m$ times. 
		Our specific architecture is:
		\[
		z_i = [x,\ y_{i-1}]^\top \quad \text{and} \quad y_i = \gamma \cdot \text{SiLU}(W z_i + b)
		\]
		The equations of our specific RNN at step \(i\) are:
		
		\begin{align}
			\nonumber
			\mathbf{h}_i &= \text{SiLU}\left( \mathbf{W}_h \cdot x + \mathbf{U}_h \cdot y_{i-1} + \mathbf{b}_h \right) \\
			\label{eq:RNN}
			y_i &= \mathbf{W}_y \cdot \mathbf{h}_i 
		\end{align}
		With the constrained parameters:
		\begin{equation}
			\label{eq:RNN_parameters}
			\mathbf{W}_h = \beta^{k} \begin{bmatrix} +1 \\ -1 \\ +1 \\ -1 \end{bmatrix}, \ \mathbf{U}_h = \beta^{k} \begin{bmatrix} +1 \\ -1 \\ -1 \\ +1 \end{bmatrix} \ \mathbf{W}_y = \frac{\beta^{-2k}}{4K(a)} \begin{bmatrix} +1 \\ +1 \\ -1 \\ -1 \end{bmatrix},
		\end{equation}
		and $\mathbf{b}_h = \begin{bmatrix} a & a & a & a \end{bmatrix}.$
		
		The initial state is \(y_0 = 1\).
		\(P_{m,k}\) is indeed a shallow RNN (a $2m$ hidden layer of $4$ neurons per step) whose parameters are fixed, shared, and not learned in this context, as they were chosen analytically to compute a specific function (\(x^m\)) rather than learned from data.
	\end{proof}
	\noindent We implement in practice an approximation of \(x^7\) with a recurrent neural network Eq.~\eqref{eq:RNN} (parameters: \(k=3\), \(a=0\), \(\beta=0.27\) as per Eq.~\eqref{eq:RNN_parameters}). The Figure~\ref{fig7} reveals that each hidden layer approximates a monomial of increasing degree.
	\begin{figure}[h!]
		\centering
		\includegraphics[width=\linewidth, height=6cm, keepaspectratio]{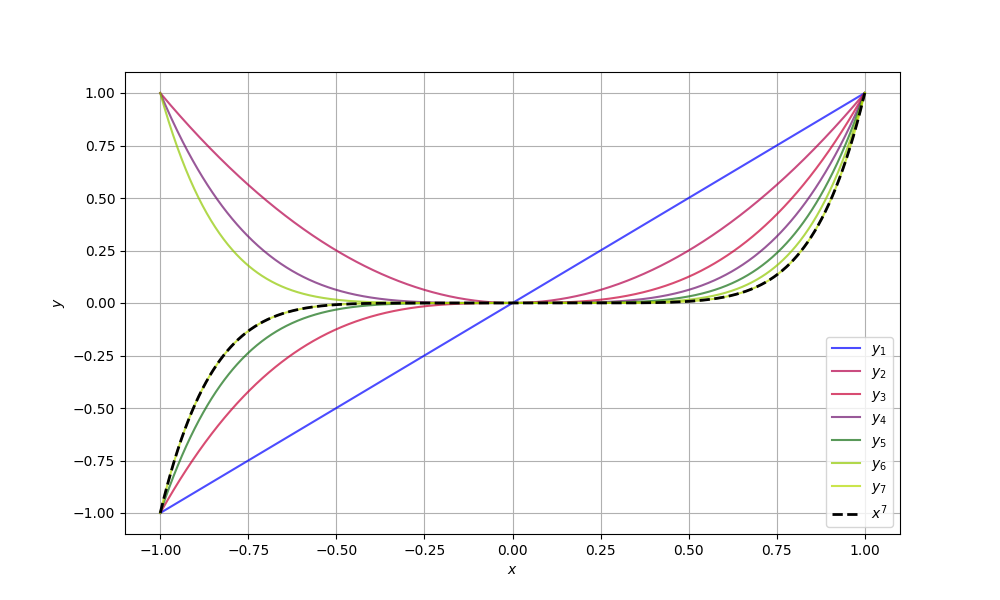}
		\caption{Approximation of the $x^7$ function by a recurrent neural network \eqref{eq:RNN} with parameters given by Eq.~ \eqref{eq:RNN_parameters}($k=3$, $a=0$ and $\beta=0.27$).}
		\label{fig7}
	\end{figure}

	\subsection{Approximation of Arbitrary Polynomials}
	
	The efficient approximation of monomials established in the previous theorems directly extends to arbitrary polynomials. The following result shows that any polynomial can be approximated by a SiLU network with exponential accuracy, where the network depth scales linearly with the polynomial degree while the width remains constant.

	\begin{theo}
		\label{theo5}
		Let $N(x) = \sum_{m=0}^M a_m x^m$ be a polynomial of degree $M$. For every $B > 0$ and target accuracy $\varepsilon \in (0,1)$, there exist constants $C_{M} > 0$, $\omega_{M} > 1$ and a family of SiLU neural network $\{\mathcal{N}_{M,k}\}_{k\in\mathbb{N}}$, with at most $2M$ hidden layers with $4$ neurons each, such that for $k = \left\lceil\frac{\ln (C'_M \cdot \max_m |a_m|) - \ln \varepsilon}{2\ln \omega'_{M}}\right\rceil$  
		\[
		\sup_{|x| \leq B} |\mathcal{N}_{M,k}(x) - N(x)| \leq \varepsilon
		\]
		The network has depth $\mathcal{O}(M)$ and size $\mathcal{O}(1)$.
	\end{theo}
	\begin{proof}
		For each \(m=0,\dots,M\), let \(P_{m,k}\) be the SiLU network approximating \(x^m\) on \([-B,B]\) with error \(\le C_m\omega_m^{-2k}\), depth \(\mathcal{O}(m)\), and constant size (Theorem~\ref{theo3}).
		
		Define \(\mathcal{N}_{M,k}(x)=\sum_{m=0}^M a_m P_{m,k}(x)\). Then
		\begin{align*}
			|\mathcal{N}_{M,k}(x)-N(x)|
			&\le \sum_{m=0}^M |a_m|\,|P_{m,k}(x)-x^m|\\
			&
			\le \Bigl(\max_m |a_m|\Bigr) \sum_{m=0}^M C_m\omega_m^{-2k}.
		\end{align*}
		Setting \(C'_M = (\max_m |a_m|)\sum_{m=0}^M C_m\) and \(\omega'_M = \min_m\omega_m\) yields
		\[
		|\mathcal{N}_{M,k}(x)-N(x)| \le C'_M \omega_{M}^{'-2k}.
		\]
		
		The networks \(P_{m,k}\) can be evaluated jointly in a single recurrent structure of depth \(\mathcal{O}(M)\) and constant width; adding the final linear combination preserves these bounds. Hence \(\mathcal{N}_{M,k}\) has depth \(\mathcal{O}(M)\), constant size, and achieves error \(\varepsilon\) for 
		$$k = \left\lceil\frac{\ln (C'_M \cdot \max_m |a_m|) - \ln \varepsilon}{2\ln \omega'_{M}}\right\rceil$$.
	\end{proof}
	
	\noindent In the same spirit, to approximate \(x^2+x+1\) using the neural network \(\mathcal{N}_{2,k}\) with parameters \(k=3\), \(a=0\), \(\beta=0.27\), We employ the network \(P_{2,k}\) to approximate the monomial \(x^2\). According to Corollary~\ref{coro1}, since \(P_{2,k}\) is an RNN, its hidden and output layers approximate the constant function \(1\), the identity function \(x\), and $x^2$ respectively. Consequently, the	sum of the outputs of these layers accurately approximates the polynomial \(x^2+x+1\) (Fig~\ref{fig8}).

	\begin{figure}[htb]
		\centering
		\includegraphics[width=\linewidth, height=6cm, keepaspectratio]{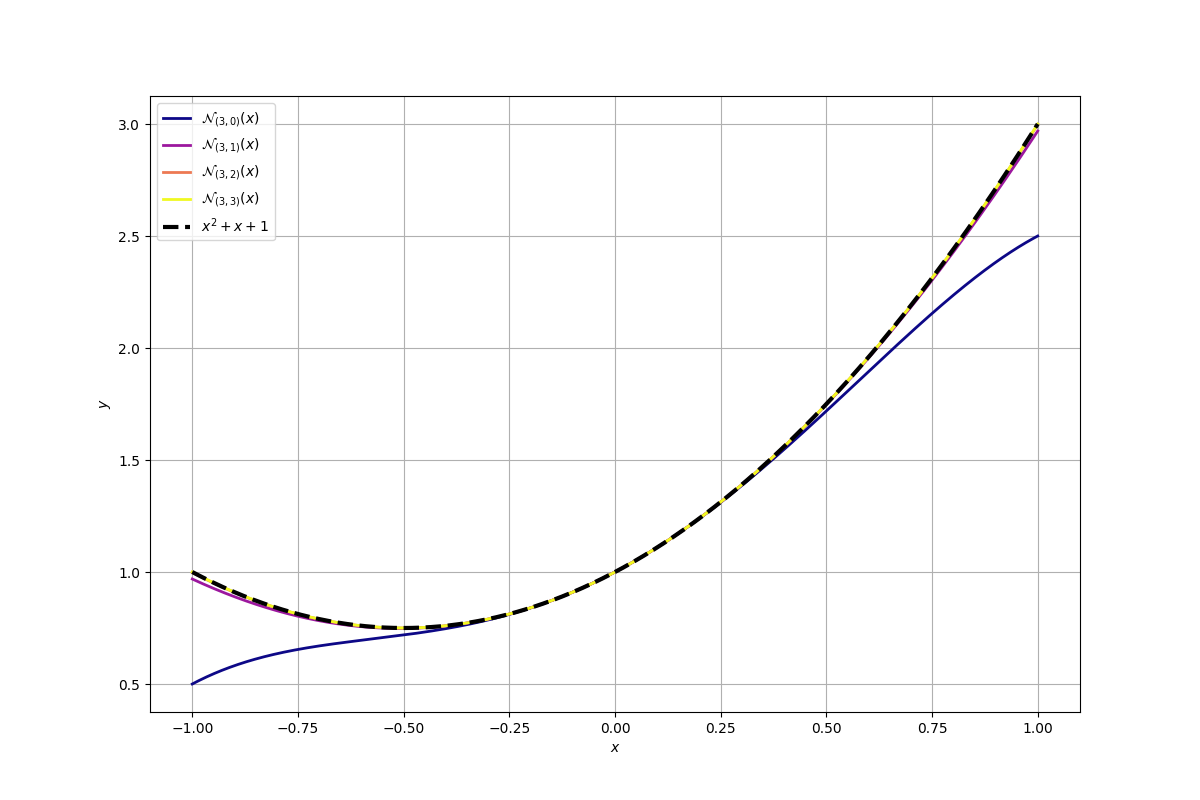}
		\caption{Approximation of the function \(x^2+x+1\) by the NN \(\mathcal{N}_{M,k}\) (Corollary \ref{coro1}).}
		\label{fig8}
	\end{figure}

	\begin{rem}
		In the case where the parameters \(c_i\) are to be learned, if we have a dataset \((x_j, y_j), j = 1, 2, \cdots, n\) such that \(y_j = N(x_j)\), then we can approximate the polynomial \(y_j\) at \(x_j\) by:
		\[
		\hat{y}_j = \begin{bmatrix}
			P_{0,k}(x_j) & P_{1,k}(x_j) & \cdots & P_{M,k}(x_j)
		\end{bmatrix} \mathbf{c}
		, 
		\]
		with $\mathbf{c}^\top = [c_0\  c_1\ \ldots,\ c_M].$
		The predicted $\mathbf{\hat{y}}\in\mathbb{R}^{n}$ output over the entire dataset can be written as a matrix product:
		\[
		\mathbf{\hat{y}} = \mathbf{P} \mathbf{c},
		\]
		where
		\[
		\mathbf{P} = \begin{bmatrix}
			P_{0,k}(x_1) & P_{1,k}(x_1) & \cdots & P_{M,k}(x_1) \\
			P_{0,k}(x_2) & P_{1,k}(x_2) & \cdots & P_{M,k}(x_2) \\
			\vdots & \vdots & \ddots & \vdots \\
			P_{0,k}(x_n) & P_{1,k}(x_n) & \cdots & P_{M,k}(x_n)
		\end{bmatrix}\in \mathbb{R}^{n \times (M+1)}.
		\]
		and $P_{0,k}(x_j)=1$, for all  $j\in\{1, 2, \cdots, n\}$.
		
		Obtaining the optimal parameters is subject to minimizing the mean squared error:
		\[
		\|\mathbf{\hat{y}} - \mathbf{P} \mathbf{c}\|_2.
		\]
	\end{rem}
	
	The efficient approximation of monomials established above serves as a fundamental building block for approximating more general function classes. We now extend these results to continuous functions and then to functions with higher regularity, deriving corresponding approximation rates for SiLU networks.
	
	\subsection{Approximation of Continuous Functions}
	
	This subsection establishes approximation rates for arbitrary continuous functions using SiLU networks. Building on the efficient polynomial approximation results from Section~\ref{sec3}, we present two methods: (i)a step-function approach whose complexity is governed by the modulus of continuity of the target function, and (ii)direct polynomial approximation via the Weierstrass theorem.
	
	The following theorem formalizes the step-function approach, which provides explicit depth and size bounds in terms of the desired accuracy $\varepsilon$ and the function's regularity.

	Let \( K \subset \mathbb{R} \) be a compact set. We denote by \( C(K) \) the space of continuous functions on \( K \), equipped with the uniform norm
	\[
	\| f \|_\infty = \sup_{x \in K} |f(x)|.
	\]
	
	We define, for \( \alpha \in \mathbb{R} \), and \( \tau > 0 \), \( \kappa > 0 \), the function:
	\[
	\phi^+_{\alpha,\kappa, \tau}(x)
	= w\bigg[ \text{SiLU}\left( \kappa\left(1 - 2w\, \text{SiLU}\left(  \frac{\kappa(\alpha - x)}{\tau} \right) \right) \right) \bigg]
	\]
	where \( w = 1/\kappa \).
	
	Finally, for \( N \in \mathbb{N}^* \), let
	\begin{align*}
		\mathcal{F}_N = \bigg\{ \sum_{i=1}^N c_i \, \phi_{\alpha_i, \beta_i,\tau, \kappa}(x) \ \bigg| \ c_i \in \mathbb{R}, \beta_i -\alpha_i,\ \tau, \ \kappa \in \mathbb{R}_+\bigg\},
	\end{align*}
	and $ 
	\mathcal{F} = \bigcup_{N \ge 1} \mathcal{F}_N$,
	such that
	\[
	\phi_{\alpha,\beta,\tau,\kappa}(x) = \phi^+_{\alpha,\tau,\kappa}(x) - \phi^+_{\beta,\tau,\kappa}(x)
	\]
	
	We have the following theorem:
	\begin{theo}
		\label{theo6}
		The set \( \mathcal{F} \) is dense in \( \big( C(K),\, \| \cdot \|_\infty \big) \).
	\end{theo}
	In order to prove this theorem, let us take a closer look at these two propositions.
	\begin{pro}
		\label{pro1}
		Let \( a < b \). For every \( \delta > 0 \), there exists \( \tau > 0 \), \( \kappa_0 > 0 \) such that for all \( \tau \in (0,\tau_0) \), \( \kappa \ge \kappa_0 \),
		\[
		\sup_{x \in \mathbb{R}} \left| \phi_{a,b,\tau,\kappa}(x)- \mathbf{1}_{[a,b]}(x) \right| < \delta.
		\]
		$\phi_{a,b,\tau,\kappa}(x)$ is a SiLU-activated neural network, with $2$ hidden layers, width $2$ with total number of parameters \( 11 \).
	\end{pro}
	\begin{proof}
		Define \(f_\kappa(x)=\kappa^{-1}\mathrm{SiLU}(\kappa x)=x\sigma(\kappa x)\). We show uniform convergence to \((x)_+=\max(0,x)\) on compacts.
		
		For \(x\ge0\): \(|f_\kappa(x)-x|=x(1-\sigma(\kappa x))\le \frac{y e^{-y}}{\kappa}\le\frac{1}{e\kappa}\) with \(y=\kappa x\).
		
		For \(x\le0\): \(|f_\kappa(x)|=|x|\sigma(\kappa x)\le\frac{y e^{-y}}{\kappa}\le\frac{1}{e\kappa}\) with \(y=-\kappa x\).
		
		Hence \(\sup_{|x|\le R}|f_\kappa(x)-(x)_+|\le (e\kappa)^{-1}\to0\) as \(\kappa\to\infty\).
		
		Setting \(f_\kappa(-x/\tau)\) gives a pointwise approximation of the Heaviside step:
		$$
		\phi_{a,\tau,\kappa}^+(x)=1-\frac{2}{\kappa}\mathrm{SiLU}\left(\frac{\kappa(a-x)}{\tau}\right)
		\xrightarrow[\kappa\to\infty, \tau\to 0^+] {} 
		\begin{cases}
			1,\ a \le x\\
			0,\ a>x
		\end{cases}$$ 
		
		i.e., \(\phi_{a,\tau,\kappa}^+\to\mathbf{1}_{[a,\infty)}\). Therefore
		\[
		\phi_{a,\tau,\kappa}^+(x)-\phi_{b,\tau,\kappa}^+(x)
		\xrightarrow[\kappa\to\infty, \tau\to 0^+]{}
		\mathbf{1}_{[a,b]}(x),
		\]
		which is exactly the claimed limit.
	\end{proof}
	\noindent The theoretical convergence shown above is confirmed numerically in Fig.~\ref{fig9}, which displays the approximation of \(\mathbf{1}_{[1,4]}(x)\) using \(\phi^+_{1,\tau,\kappa}(x) - \phi^+_{4,\tau,\kappa}(x)\).
	
	\begin{figure}[htb]
		\centering
		\includegraphics[width=\linewidth, height=6cm, keepaspectratio]{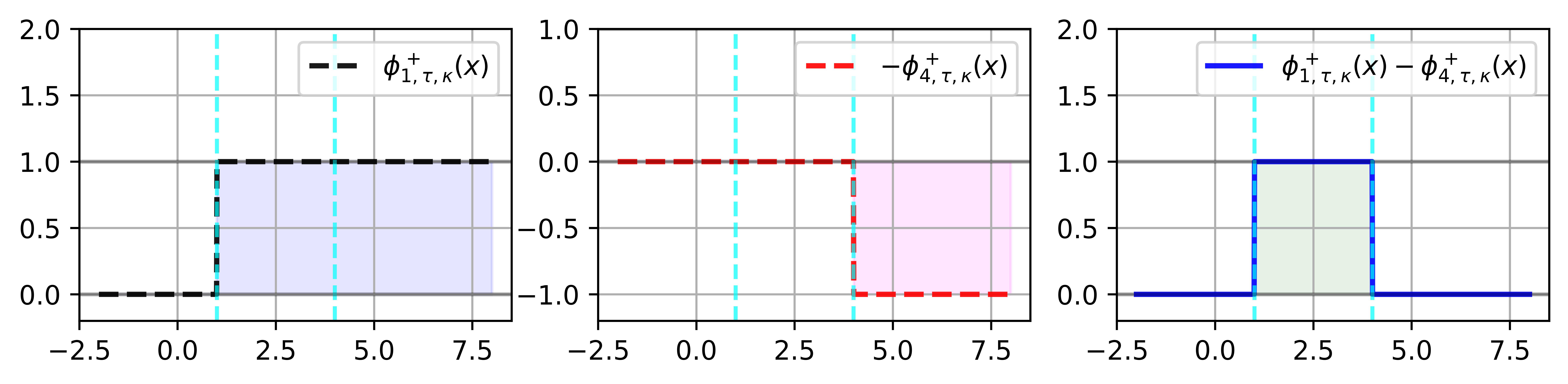}
		\caption{Approximation of the function \(\mathbf{1}_{[1, 4]}(x)\) by the function \(\phi^+_{1,\tau,\kappa}(x) -\phi^+_{4,\tau,\kappa}(x)\) (Proposition~\ref{pro1}).}
		\label{fig9}
	\end{figure}
	We say that a function \( \psi : K \to \mathbb{R} \) is a non-identically zero step function if there exists a finite partition of \( K \) into pairwise disjoint intervals \( I_1, \dots, I_m \) with non-empty interior, and real numbers \( d_1, \dots, d_m \) not all zero, such that:
	\[
	\psi(x) = \sum_{j=1}^m d_j \mathbf{1}_{I_j}(x).
	\]
	
	\begin{pro}[Approximation of Step Functions]
		\label{pro2}
		Let \( \psi : K \to \mathbb{R} \) be a non-identically zero step function:
		\[
		\psi = \sum_{j=1}^m d_j \mathbf{1}_{I_j},
		\]
		where the \( I_j \) are pairwise disjoint intervals with non-empty interior, \( d_j \in \mathbb{R} \).
		Then, for every \( \eta > 0 \), there exists \( N \in \mathbb{N}^* \), choices of \( \alpha_i, \beta_i, c_i, \tau, \kappa \) such that
		\[
		\left\| \psi - \sum_{i=1}^N c_i \phi_{\alpha_i,\beta_i,\tau,\kappa} \right\|_\infty < \eta.
		\]
	\end{pro}
	
	\begin{proof}
		We take \( N = m \), \( \alpha_i, \beta_i \) such that \( I_i = [\alpha_i, \beta_i] \).
		By Proposition~\ref{pro1}, there exist $\tau_0$
		and $\kappa_0$ two strictly positive real numbers such that for all $\tau_0> \tau > 0$ and $\kappa > \kappa_0$, such that
		\[
		\| \phi_{\alpha_i,\beta_i,\tau,\kappa} - \mathbf{1}_{I_i} \|_\infty < \delta \quad \forall i.
		\]
		We set \( c_i = d_i \). Then, for \( \delta = \frac{\eta}{m \cdot \max_j \{|d_j|, |c_j|\}}  \)
		\begin{align*}
			\Bigg\|  \psi - & \sum_{i=1}^m c_i \phi_{\alpha_i,\beta_i,\tau,\kappa} \Bigg\|_\infty\\
			&\le \sum_{i=1}^m  \max_j \{|d_j|, |c_j|\} \, \| \mathbf{1}_{I_i} - \phi_{\alpha_i,\beta_i,\tau,\kappa} \|_\infty\\
			&
			\le m \cdot \max_j \{|d_j|, |c_j|\} \cdot\delta\\
			& \le \eta.
		\end{align*}
		
	\end{proof}
	
	\begin{proof}[Proof of theorem~\ref{theo6}]
		
		Let \( f \in C(K) \), \( \varepsilon > 0 \). Since \( K \) is compact and \( f \) is continuous, \( f \) is uniformly continuous on \( K \).
		
		There exists \( \eta > 0 \) such that:
		\[
		\forall x,y \in K,\ |x - y| < \eta \Rightarrow |f(x) - f(y)| < \frac{\varepsilon}{2}
		\]
		
		We partition the compact \( K \) into $m$ intervals \( I_1, \dots, I_m \) of diameter less than \(\eta \). For each \( I_j \), we choose \( x_j \in I_j \) and define:
		\[
		\psi(x) = f(x_j) \quad \text{for } x \in I_j
		\]
		
		Then for all \( x \in K \), if \( x \in I_j \), we have \( |x - x_j| < \eta \) so:
		\[
		|f(x) - \psi(x)| = |f(x) - f(x_j)| < \frac{\varepsilon}{2}
		\]
		If we set \(\delta = \frac{\varepsilon}{4m \cdot \max_j \{|d_j|, |c_j|\}} \), we deduce the approximation error of the continuous function $f$
		\[
		\|f - \sum_{i=1}^N c_i \, \phi_{\alpha_i, \beta_i,\kappa, \tau}\|_\infty \leq \|f - \psi\|_\infty + \|\psi - \sum_{i=1}^N c_i \, \phi_{\alpha_i, \beta_i,\kappa, \tau}\|_\infty \leq \frac{\varepsilon}{2}+\frac{\varepsilon}{2}= \varepsilon
		\]
		Therefore \( \mathcal{F} \) is dense in \( C(K) \).
	\end{proof}
	
	\begin{coro}[Fast Approximation of Continuous Functions by SiLU Neural Networks]
		\label{coro2}
		For \( f \in C([a,b]) \) with modulus of continuity \( \omega_f \), and for \( \varepsilon > 0 \), there exists a SiLU-activated neural network, with $2$ hidden layers and width $2N$
		\[
		f_\varepsilon(x) = \sum_{i=1}^N c_i \phi_{\alpha_i,\beta_i,\tau,\kappa}(x)
		\]
		such that:
		\[
		N = \mathcal{O}\left( \frac{b-a}{\omega_f^{-1}(\varepsilon)} \right), \quad \|f - f_\varepsilon\|_\infty < \varepsilon
		\]
		
		The total number of parameters is \( 10N + 1 \).
	\end{coro}
	
	The second approach leverages the classical Weierstrass approximation theorem combined with our efficient polynomial implementation (Theorem~\ref{theo4}).
	
	\begin{theo}[Polynomial Approximation of Continuous Functions by SiLU Neural Networks]
		\label{theo7}
		Let \( f \in C([-B,B]) \). For every \( \varepsilon \in (0,1) \), there exists a SiLU neural network \( \mathcal{N}_{M,k} \) such that:
		\[
		\sup_{|x| \leq B} |f(x) - \mathcal{N}_{M,k}(x)| \leq \varepsilon
		\]
		$\mathcal{N}_{M,k}$ is a NN with SiLU activation like in Theorem~\ref{theo4}.
	\end{theo}
	
	To prove this theorem, we recall the following lemma:
	
	\begin{lem}
		\label{lem1}
		Let \( f \) be a continuous function from \( [a, b] \) to \( \mathbb{R} \).
		For every \( \delta > 0 \), there exists a polynomial function \( P_f \) with real coefficients such that for all \( x \) in \( [a, b] \), \( |f(x) - P_f(x)| \leq \delta \).
	\end{lem}
	\begin{proof}
		
		Let \( f \) be a \( C^0 \) function on \( [-B, B] \). Lemma \ref{lem1} guarantees the existence of a polynomial function \( P_f \) with real coefficients such that for all \( x \) in \( [-B, B] \), \( |f(x) - P_f(x)| \leq \delta= \frac{\varepsilon}{2} \leq \frac{1}{2} \). 
		According to Theorem~\ref{theo4}, there exists a SiLU neural network \( \mathcal{N}_{M,k} \) such that 
		\begin{align*}
			\sup_{|x| \leq B}|P_f(x)-\mathcal{N}_{M,k}(x)| \leq C_M \cdot \max_m |a_m| \cdot \omega_4^{-2k}
		\end{align*}
		
		\begin{align*}
			|f(x)-\mathcal{N}_{M,k}(x)|
			\leq |f(x)-P_f(x)|+|P_f(x)-\mathcal{N}_{M,k}(x)|
		\end{align*}
		\begin{align*}
			\sup_{|x| \leq B} |f(x)-\mathcal{N}_{M,k}(x)|	\leq \delta + C_M \cdot \max_m |a_m| \cdot \omega_M^{-2k}
		\end{align*}
		If we set  
		\[
		k = \left\lceil -\frac{\ln \frac{\varepsilon}{2C_M \cdot \max_m |a_m|}}{2 \ln \omega_M} \right\rceil,
		\]
		we obtain 
		\[
		\sup_{|x| \leq B} |f(x) - \mathcal{N}_{M,k}(x)| \leq \varepsilon.
		\]
		
	\end{proof}
	
	The following result shows that these networks can approximate any Sobolev function with a complexity that depends on the required accuracy, and establishes the optimal convergence rates depending on the regularity of the target function and the dimension of the space. Let \( B > 0 \), \( d, n \in \mathbb{N}^* \), and consider the open ball of Sobolev space:
	Functions belonging to the Sobolev class are defined as follows.
	
	Let \( \Omega = [-B, B]^d \), with $B\leq 1$, be a bounded domain. In this section, vectors are denoted in bold. For a multi-index \( \balpha = (\alpha_1, \dots, \alpha_d) \in \mathbb{N}^d \), denote \( |\balpha| = \alpha_1 + \cdots + \alpha_d \) and $D^{\balpha} f$
	a weak derivative of $f.$
	The Sobolev space \( W^{n,\infty}(\Omega) \) is defined by
	\[
	W^{n,\infty}(\Omega) = \big\{ f \in L^\infty(\Omega) \; \big|\forall \balpha  : \ |\balpha|\le n,  D^{\balpha} f \in L^\infty(\Omega) \}.
	\]
	It is equipped with the norm
	\[
	\| f \|_{W^{n,\infty}(\Omega)} = \max_{|\balpha| \le n} \| D^{\balpha} f \|_{L^\infty(\Omega)},
	\]
	where
	\[
	\| g \|_{L^\infty(\Omega)} = \operatorname{ess\,sup}_{x \in \Omega} |g(x)|.
	\]
	We consider the unit ball of this space:
	\[
	\mathcal{F}_{n,d} = \big\{ f \in W^{n,\infty}([-B,B]^d) : \| f \|_{W^{n,\infty}} \le 1 \big\}.
	\]
	
	\begin{theo}
		\label{theo8}
		For any $f\in \mathcal{F}_{n,d}$ and \( \varepsilon \in (0,1) \), there exists a neural network \( \widetilde{f}_\varepsilon \) with SiLU activation function such that:
		\[
		\sup_{x \in [-B, B]^d} |f(x) - \widetilde{f}_\varepsilon(x)| \leq \varepsilon,
		\]
		with a complexity of:
		\[
		\text{size}(\widetilde{f}_\varepsilon) = \mathcal{O}\left( \varepsilon^{-d/n}  \right), \quad \text{depth}(\widetilde{f}_\varepsilon) = \mathcal{O}\left( 1 \right)=c(d,n,B),
		\]
		where the implicit constants depend on \( d \), \( n \), and \( B \).
	\end{theo}
	\begin{proof}
		The proof proceeds in three key steps: a discretization of the space into small cubes where the function is locally approximated by its Taylor polynomial; an aggregation of these local approximations into a global; and finally, the efficient implementation of this construction using a SiLU neural network, whose complexity we analyze.
		
		Partition the domain \( [-B, B]^d \) into $M^d$ cubes of side length \( h = \frac{2B}{M} \), where \( \bc_\bj = (c_{j_1},c_{j_2}, \cdots, c_{j_d}) \) is the center of \( \bC_\bj \).
		For each cube $\bC_\bj$, define a function $\psi_\bj : \mathbb{R}^d \to [0,1]$ that is approximately $\mathbf{1}_{\bC_{\bj}}$. Define
		
		\[
		\psi_\bj(\bx)=\prod_{i=1}^d \mathbf{1}_{\bC_{j_i}}(x_i),
		\]
		
		It follows that 
		$$
		\sum_{\bj}\psi_\bj(\bx)\equiv1,\qquad \bx\in[-B,B]^d.
		$$
		
		For each cube \( C_\bj \) and for any \( f \in \mathcal{F}_{n,d} \), we consider the Taylor polynomial of order \( n-1 \) at the point \( \bx=\bc_\bj \):
		\[
		T_{\bj}(\bx) = \sum_{|\balpha| \leq n-1} \frac{D^{\balpha} f(\bc_{\bj})}{\balpha!} (\bx - \bc_{\bj})^{\balpha}.
		\]
		We define the global approximant of the function $f$
		\[
		f_\varepsilon(\bx) = \sum_{\bj} \psi_\bj(\bx) T_\bj(\bx).
		\]
		
		\begin{equation*}
			\begin{split}
				| f(\bx) - f_\varepsilon(\bx) |
				&= \Bigl| \sum_{\bj} \psi_\bj(\bx)\bigl(f(\bx) - T_\bj(\bx)\bigr) \Bigr| \\
				&\leq \sum_{\bj\in C_\bj} \bigl|f(\bx) - T_\bj(\bx)\bigr|,
				\quad (\psi_\bj(\bx)\leq 1) \\
				&\leq 2^d \max_{\bj\in C_\bj} \bigl|f(\bx) - T_\bj(\bx)\bigr| \\
				&\leq 2^d\frac{h^{n}}{n!2^{n}} 
				\max_{|\balpha|=n} \operatorname*{ess\,sup}_{\bx\in[-B,B]^d}
				|D^{\balpha} f(\bx)| \\
				&\leq 2^d\frac{h^{n}}{n!2^{n}} 
			\end{split}
		\end{equation*}
		The last three inequalities hold because respectively: At each point, at most 
		$2d$ 
		functions $\psi_\bj$ are nonzero; for \( \bx \in C_\bj \), we have \( |x_i - c_{j_i}| \leq \frac{h}{2} \) for all \( i \),
		and $\text{ess sup}_{\bx\in[-B,B]^d}\left|D^{\balpha} f(\bx) \right|\leq1.$
		
		We want \( \|f - f_\varepsilon\|_{\infty} \leq \frac{\varepsilon}{2} \). It suffices 
		to choose:
		\[
		M = \left\lceil \frac{2B}{h} \right\rceil = \left\lceil B \left( \frac{2^{d+1}}{(n!)\varepsilon} \right)^{1/n} \right\rceil.
		\]
		
		Thus, \( M = O(\varepsilon^{-1/n}) \) and \( N_{\text{cubes}} = M^d = O(\varepsilon^{-d/n}) \).
		In total, we have constructed the polynomial function $f_\varepsilon$ that approximates $f$ within $\varepsilon/2$
		$$
		f_\varepsilon(\bx) = \sum_{\bj\in M_d}  \sum_{|\balpha| \leq n-1} \underbrace{\frac{D^{\balpha} f(\bc_{\bj})}{\balpha!} }_{C_{\bj,\balpha} }\psi_\bj(\bx)(\bx - \bc_{\bj})^{\balpha},
		$$
		with  $C_{\bj,\balpha}\leq 1$.
		According to Theorem~\ref{theo3}, we approximate each term \(\psi_\bj(\bx) (\bx - \bc_\bj)^{\balpha} =\prod_{i=1}^d \mathbf{1}_{\bC_{j_i}}(x_i) \prod_{i=1}^d (x_i - c_{j_i})^{\alpha_i} \), a product of $d+n$ factors 
		\begin{align*}
			\mathbf{1}_{\bC_{j_1}}(x_1)\times \mathbf{1}_{\bC_{j_2}}(x_2)\times \cdots\times \mathbf{1}_{\bC_{j_d}}(x_d)\times (x_1 - c_{j_1})\times\cdots 
		\end{align*}
		by the network 
		\begin{multline*}
			\widetilde G_{\mathbf{j},\boldsymbol{\alpha}}(\mathbf{x})=\widetilde{M}_k\Bigl( \phi_{a_{j_1},b_{j_1},\tau,\kappa}(x_1), 
			\widetilde{M}_k\bigl(\phi_{a_{j_2},b_{j_2},\tau,\kappa}(x_2), \ldots,\\
			\widetilde{M}_k\bigl((x_1 - c_{j_1}),\ldots\bigr) \bigr)\Bigr)
		\end{multline*}
		We assume that the approximation error for a product is $\eta$ i.e. $|\widetilde{M}_k(x,y)-xy|\leq \eta$ (depth $O(1)$ and size $O(1)$).
		
		Let us analyze the propagation of approximation errors for $(\bx - \bc_\bj)^{\balpha}.$ Denote by $\gamma_i, i\geq 2$ the approximation error at step  of $i-1$ factors with $\gamma_2=0.$ 
		\begin{align*}
			\gamma_{i+1}=\left|\widetilde{M}_k\left(u_{i+1},\widetilde{\prod_{i=1}^{i}u_i}\right)-u_{i+1}\prod_{i=1}^{i}u_i\right|
			&\leq 
			\left|\widetilde{M}_k\left(u_{i+1},\widetilde{\prod_{i=1}^{i}u_i}\right)-u_{i+1}\widetilde{\prod_{i=1}^{i}u_i}\right|
			+\left|u_{i+1}\widetilde{\prod_{i=1}^{i}u_i}-u_{i+1}\prod_{i=1}^{i}u_i\right|\\
			&\leq \begin{cases}
				\eta + \gamma_{i},\ \text{If}\ \ |u_{i+1}|= |\phi_{\alpha_{j_{i+1}},\beta_{j_{i+1}},\tau,\kappa}(x_{i+1})|\leq 1\\
				\eta + 2\gamma_{i}, \ \text{If}\ \ |u_{i+1}|=|x_{i+1} - c_{j_i}| \leq 2B\leq 2,
			\end{cases} 
		\end{align*}
		
		So the network \(\widetilde G_{\bj,\balpha}(\bx)\) approximates the term \(\psi_\bj(\bx) (\bx - \bc_\bj)^{\balpha} \) with error
		$$\|\psi_\bj(\bx) (\bx - \bc_\bj)^{\balpha}  - \widetilde G_{\bj,\balpha}(\bx)\|_\infty \le \eta(2^{n-2}-1+d).$$

		Summing over all terms \((\bj,\balpha)\) (finite number), we obtain
		\begin{align*}
			\sup_{\bx\in [-B,B]^d} \big|\widetilde f_\varepsilon(\bx)-f_\varepsilon(\bx)\big|
			&\leq\big|\sum_{\bj,\balpha} \left(C_{\bj,\balpha}\ \widetilde G_{\bj,\balpha}(\bx) - \psi_\bj(\bx) (\bx - \bc_\bj)^{\balpha}\right)\big|\\
			&  \leq\big|\sum_{\bj : \bx\in  \text{supp}\psi_\bj}\sum_{\balpha} \left(C_{\bj,\balpha}\ \widetilde G_{\bj,\balpha}(\bx) - \psi_\bj(\bx) (\bx - \bc_\bj)^{\balpha}\right)\big|\\
			&  \leq 2^d\max_{\bj : \bx\in  \text{supp}\psi_\bj}\sum_{\balpha} \left(C_{\bj,\balpha}\ \big|\widetilde G_{\bj,\balpha}(\bx) - \psi_\bj(\bx) (\bx - \bc_\bj)^{\balpha}\big|\right)\\
			&\leq 2^d d^n\eta(2^{n-2}-1+d)
		\end{align*}
		
		Then, it suffices to choose \(\eta\) such that
		\[
		\eta = \frac{\varepsilon}{2^{d+1} d^n(2^{n-2}-1+d)}.
		\]
		i.e.
		$$k \geq \left\lceil\frac{\ln C_2 - \ln \eta}{2\ln \omega_2}\right\rceil=\mathcal{O}(\ln 1/\varepsilon)$$
		to guarantee
		\[
		\sup_{\bx\in [-B,B]^d} \big|\widetilde f_\varepsilon(\bx)-f_\varepsilon(\bx)\big| \le \frac{\varepsilon}{2}.
		\]
		Thus, we show that the network $\widetilde{f}_\varepsilon$ approximates the function $f$ with error $\varepsilon$
		\begin{align*}
			\sup_{\bx\in [-B,B]^d} \big|\widetilde f_\varepsilon(\bx)-f(\bx)\big| &\le\sup_{\bx\in [-B,B]^d} \big|\widetilde f_\varepsilon(\bx)-f_\varepsilon(\bx)+f_\varepsilon(\bx)-f(\bx)\big| 
			\\
			&\le \sup_{\bx\in [-B,B]^d} \big|\widetilde f_\varepsilon(\bx)-f_\varepsilon(\bx)\big|+ \sup_{\bx\in [-B,B]^d} \big|f_\varepsilon(\bx)-f(\bx)\big| 
			\\
			&\le \frac{\varepsilon}{2}+\frac{\varepsilon}{2}\\
			&\leq \varepsilon.
		\end{align*}

		The \(\mathcal{O}(\varepsilon^{-d/n})\) terms \(\widetilde{G}_{\mathbf{j},\alpha}\) are computed in parallel over \(\mathcal{O}(1)\) layers with width \(\mathcal{O}(\varepsilon^{-d/n})\), then summed linearly via the addition of one linear layer, yielding a network \(\widetilde{f}_\varepsilon\) with total depth \(\mathcal{O}(1)=c(d,n,B)\) and size \(\mathcal{O}(\varepsilon^{-d/n})\) satisfying \(\|f - \widetilde{f}_\varepsilon\|_\infty \leq \varepsilon\).
		
	\end{proof}

	\section{Discussions and experimental results}
	\label{sec4}

	\subsection{Parameterization of the Construction}
	
	Our hierarchical construction relies on three key parameters whose optimization is crucial for performance: the scale parameter $\beta$, the recursion parameter $k$, and the symmetrization shift $a$.
	
	\begin{figure}[h!]
		\centering
		\includegraphics[width=\linewidth, height=6cm, keepaspectratio]{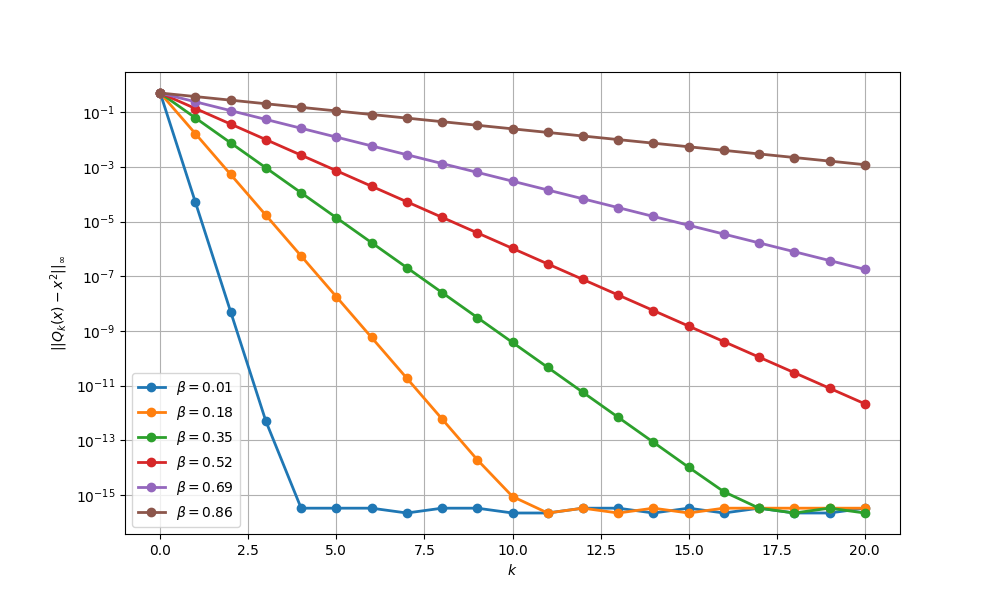}
		\caption{Error evolution as a function of $\beta$ and $k$.}
		\label{fig10}
	\end{figure}
	
	\begin{figure}[h!]
		\centering
		\includegraphics[width=\linewidth, height=6cm, keepaspectratio]{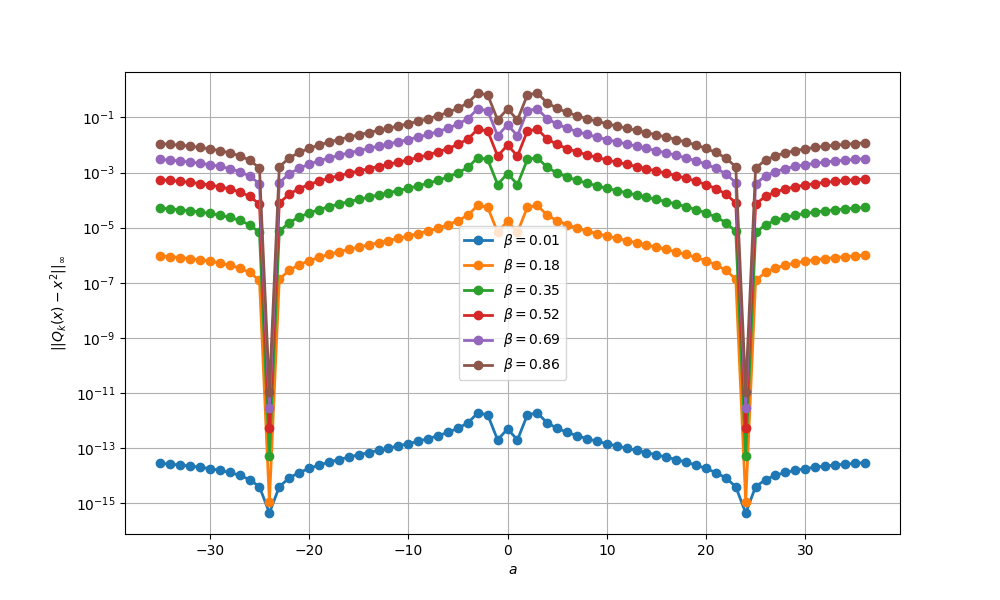}
		\caption{Error evolution as a function of $\beta$ and $a$.}
		\label{fig11}
	\end{figure}
	
	\begin{figure}[h!]
		\centering
		\includegraphics[width=\linewidth, height=6cm, keepaspectratio]{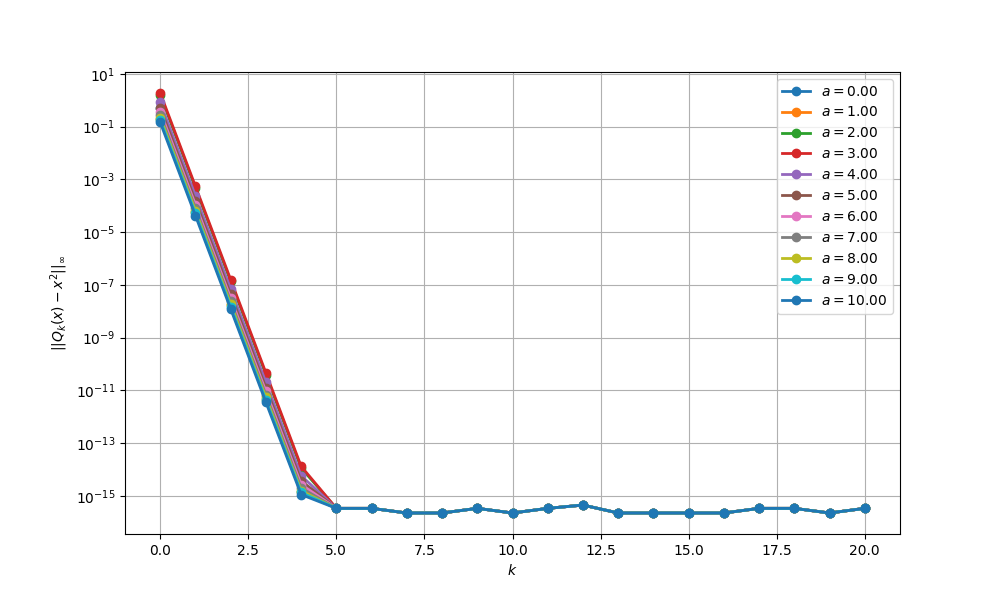}
		\caption{Error evolution as a function of the correlation between $k$ and $a$.}
		\label{fig12}
	\end{figure}
	
	\paragraph{Scale Parameter $\beta$}
	The parameter $\beta$ determines the compression rate in the recursive construction $Q_k(x) = \beta^{2k} f_a(\beta^{-k} x)$. Its optimal value results from a trade-off between convergence speed and numerical stability. A systematic analysis reveals that values that are too small (e.g., $0.01\leq \beta \leq 0.5$, Fig.~\ref{fig10}) theoretically accelerate convergence. These values of$\beta$ offer a good compromise, allowing for low values of $k$ to achieve maximum reconstruction errors on the order of $10^{-7}$.
	
	\paragraph{Recursion Parameter $k$}
	The parameter $k$, less influenced by $a$, controls the accuracy of the approximation according to the relation $Q_k(x) = \beta^{2k} f_a(\beta^{-k} x)$. Experimentally, we observe that $k = 3$ is sufficient to achieve an accuracy of $10^{-6}$ in the approximation of the function $x^2$ (Fig.~\ref{fig12}), while $k = 5$ allows reaching $10^{-15}$ (Fig.~\ref{fig10}). 
	
	\paragraph{Symmetrization Shift $a$}
	The parameter $a$ in the  function $g_a(x) = (x+a)\sigma(x+a) + (-x+a)\sigma(-x+a) - 2a\sigma(a)$ influences the dominant coefficient $K(a) = 2\sigma'(a) + a\sigma''(a)$. Numerical analysis shows that the Taylor expansion remains well‑conditioned when centered at \(a = 0\) or \(a = 1\), provided the function and the evaluation points are appropriately scaled (Fig.~\ref{fig11}).
	
	\subsection{Approximation of Continuous Functions}
	
	Our experiments (Figs. \ref{fig13}~-~\ref{fig16}) implemented networks of by theorems~\ref{theo6},~\ref{theo7} and~\ref{theo8} to approximate the following functions:
	\begin{align*}
		T_1(x) = \log(7 + x) \cos(x^3), \ 
		T_2(x) = \frac{1}{1 + e^{-x}}, \ 
		T_3(x) = (x, \sin(x), \cos(x)), \
		T_4(x, y)= \frac{\sin(\pi x)\cos (\pi x)}{2\pi}.
	\end{align*}
	\begin{figure}[h!]
		\centering
		\includegraphics[width=0.48\linewidth, height=6cm, keepaspectratio]{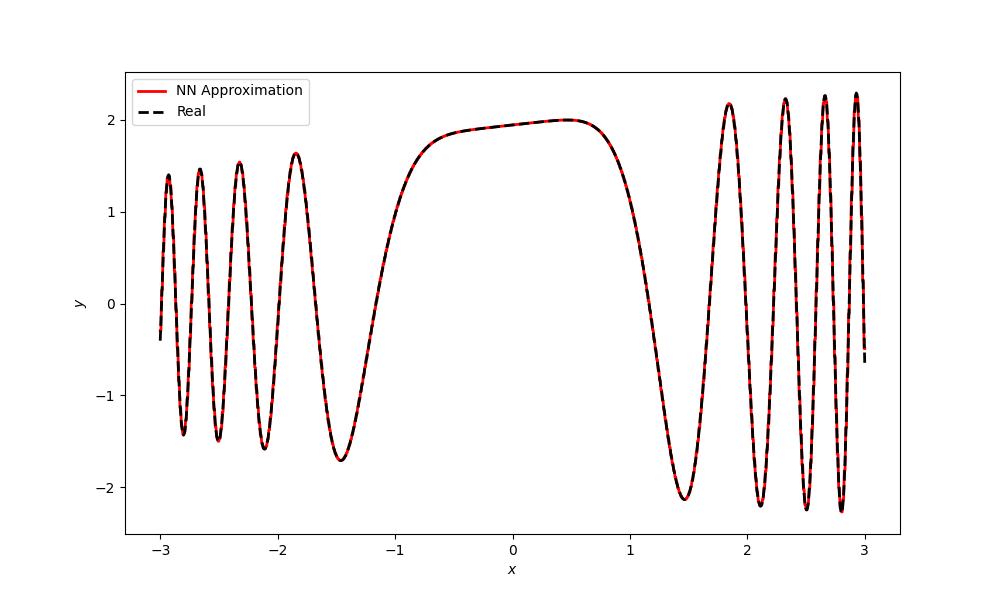} \quad
		\includegraphics[width=0.48\linewidth, height=6cm, keepaspectratio]{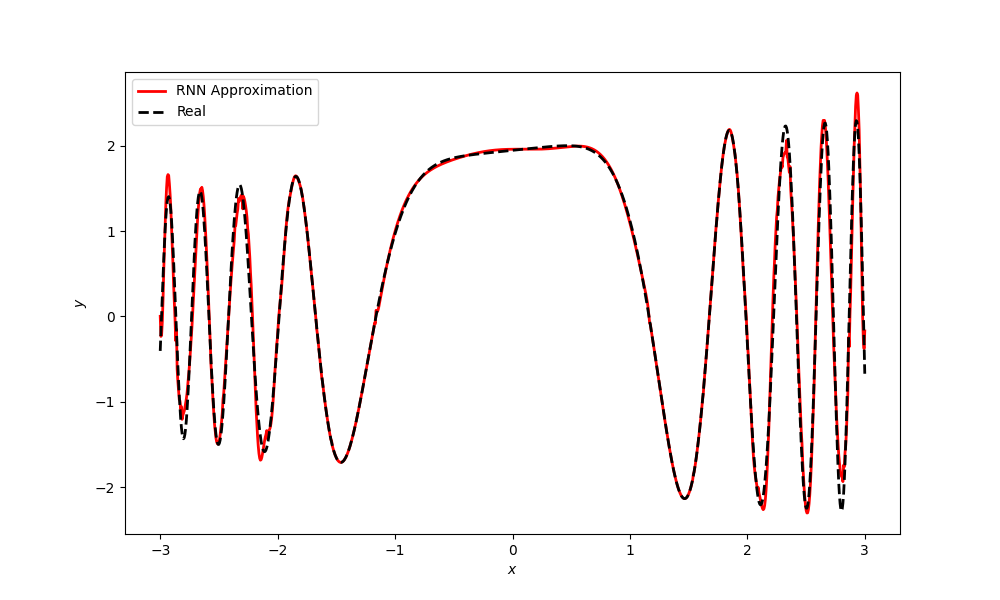}
		\caption{Approximation of the function $\log(7+x)\cos(x^3)$ by the NN of  Theorem~\ref{theo6} (in left) and Theorem~\ref{theo7} (in right).}
		\label{fig13}
	\end{figure}
	
	\begin{figure}[h!]
		\centering
		\includegraphics[width=0.48\linewidth, height=6cm, keepaspectratio]{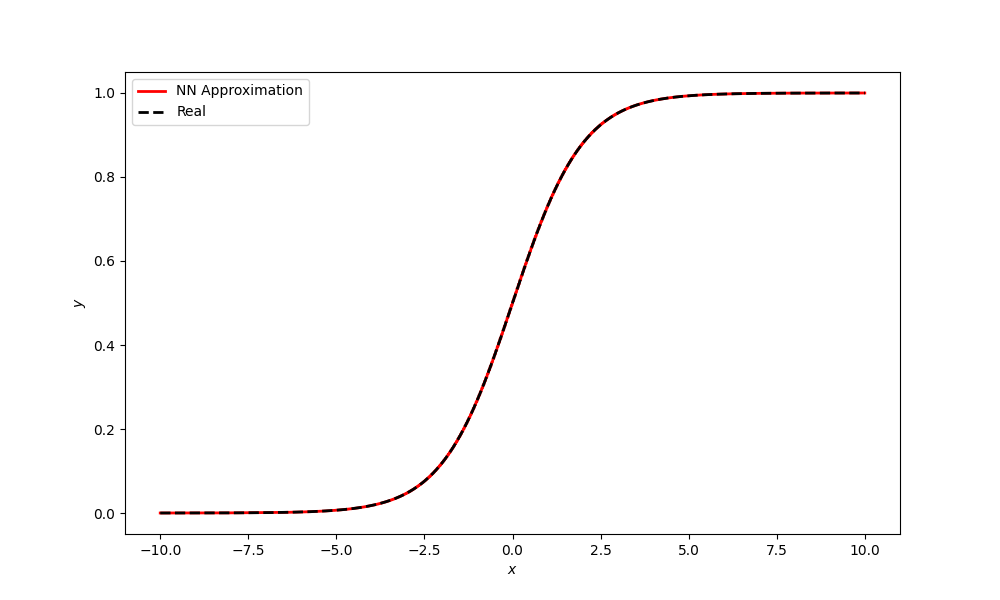}\quad
		\includegraphics[width=0.48\linewidth, height=6cm, keepaspectratio]{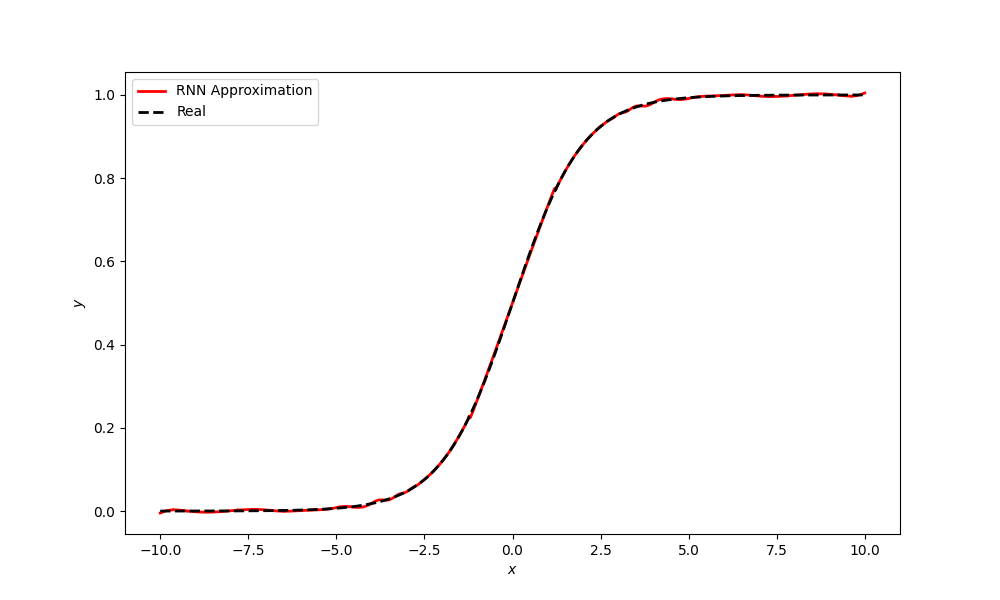}
		\caption{Approximation of the sigmoid function $\frac{1}{1+e^{-x}}$ by the NN of  Theorem~\ref{theo6} (in left) and Theorem~\ref{theo7} (in right).}
		\label{fig14}
	\end{figure}
	
	\begin{figure}[h!]
		\centering
		\includegraphics[width=\linewidth, height=6cm, keepaspectratio]{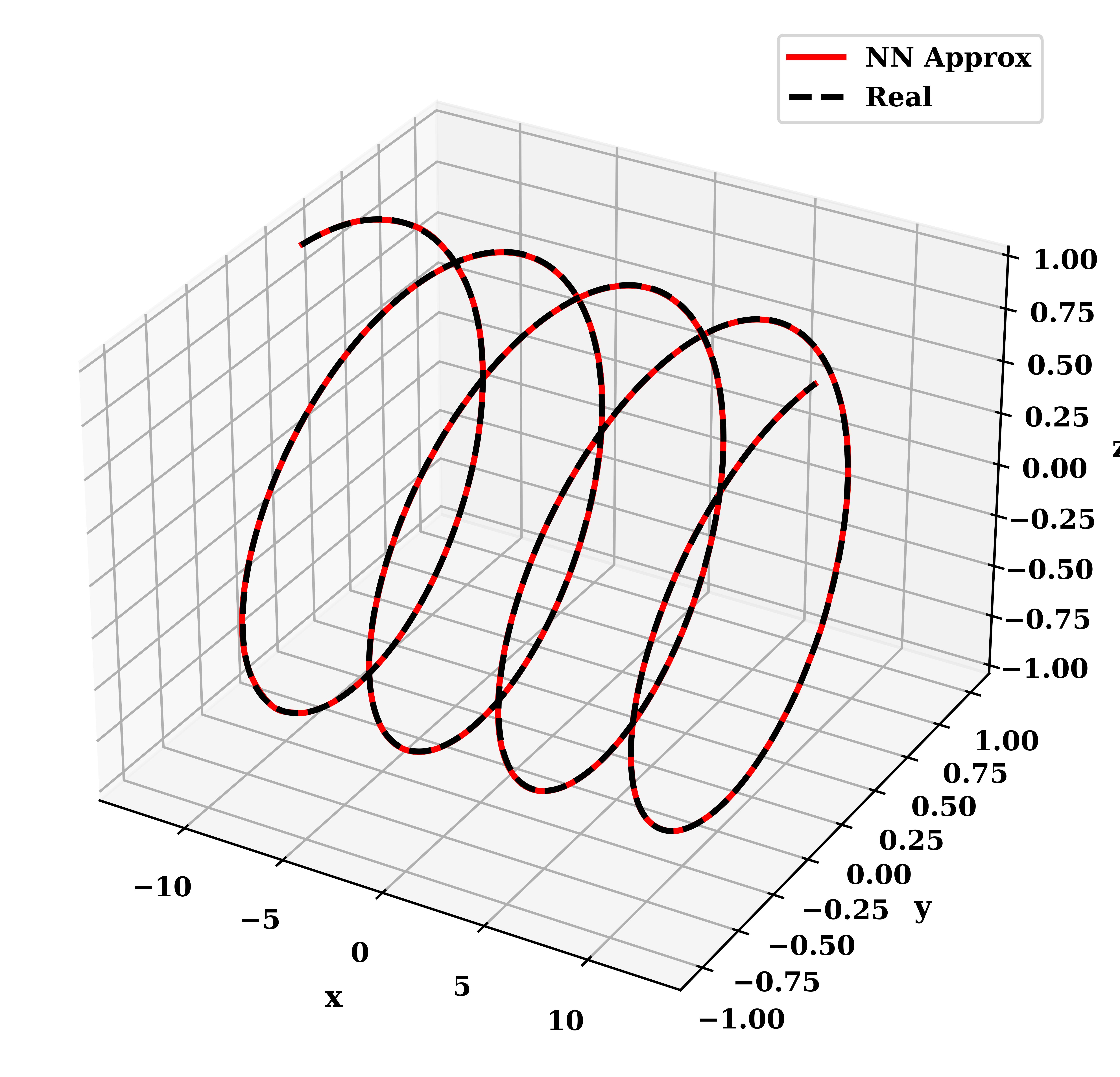}
		\caption{Approximation of the function $x \mapsto (x, \sin(x), \cos(x))$ by theorem~\ref{theo7}.}
		\label{fig15}
	\end{figure}
	\begin{figure}[h!]
		\centering
		\includegraphics[width=\linewidth, height=6cm, keepaspectratio]{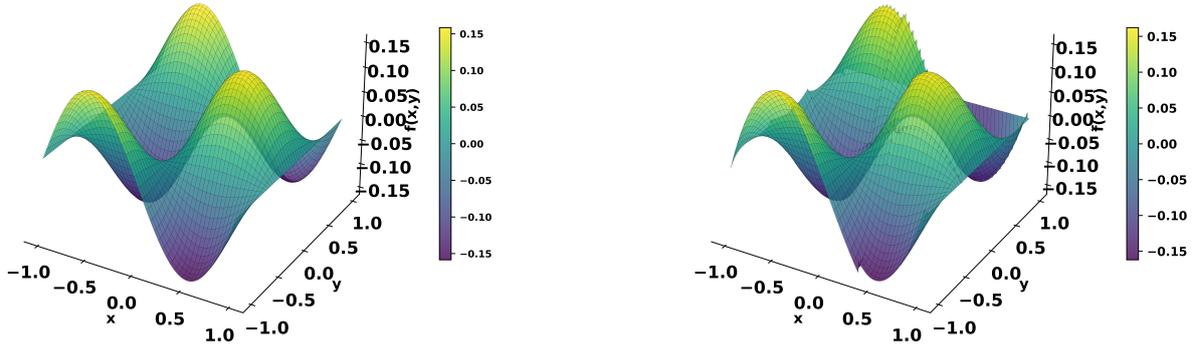}
		\caption{Approximation of the function $(x, y)\mapsto \frac{\sin(\pi x)\cos (\pi x)}{2\pi}$ (in left) by theorem~\ref{theo8} (in right).}
		\label{fig16}
	\end{figure}
	
	The proposed architectures with SiLU activation demonstrate remarkable approximation capabilities for extended classes of continuous and analytic functions. Our theoretical analysis and experimental validation establish that these models can be designed as generators of adaptive functional bases, systematically producing monomials, cross-interactions, and other fundamental functional elements through deep sequential compositions.
	\begin{figure}[h!]
		\centering
		\includegraphics[width=\linewidth, height=8cm, keepaspectratio]{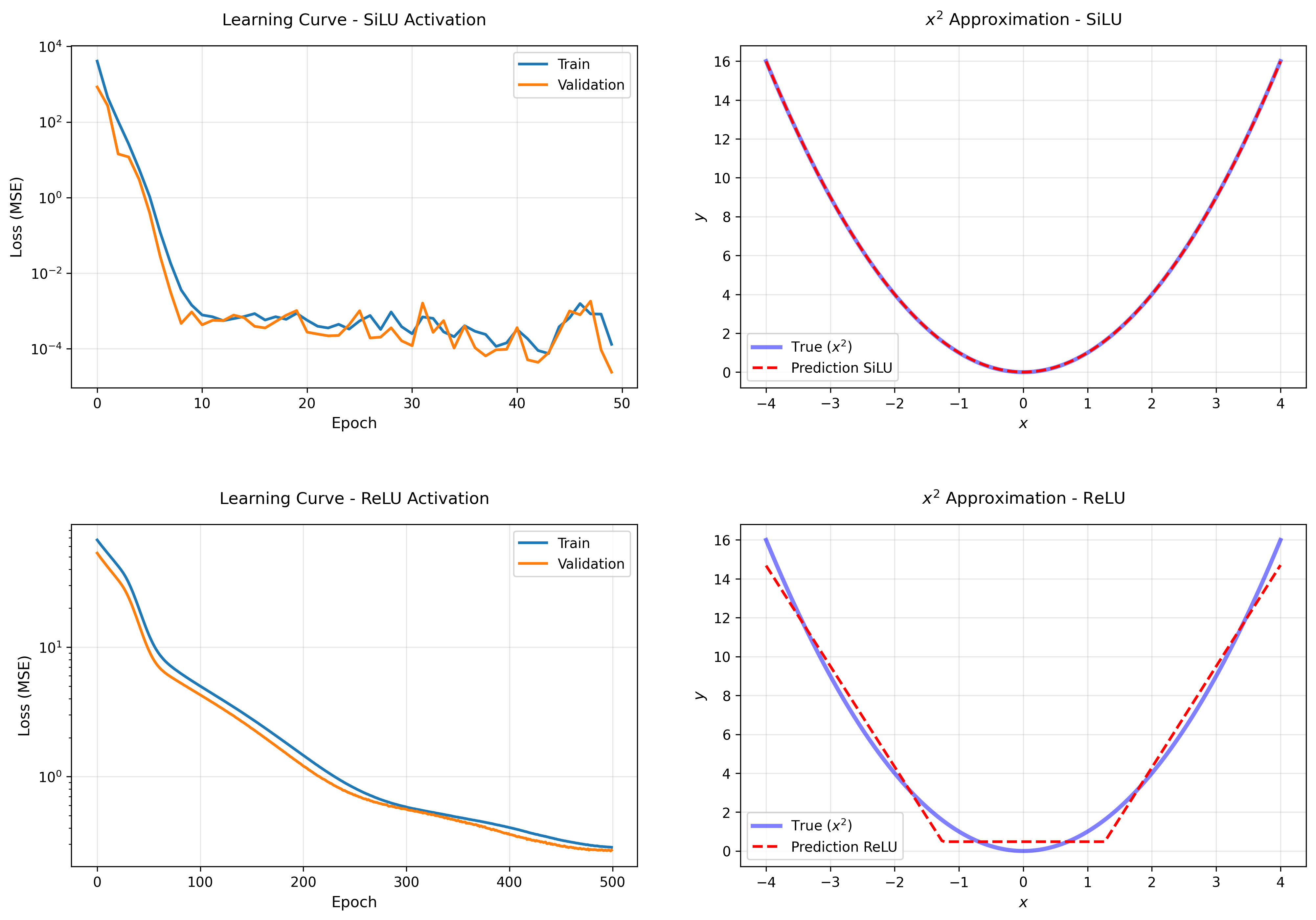}
		\caption{Using the architecture specified in Theorem~\ref{theo1} to approximate the function \(x^2\), a comparison between SiLU and ReLU activations reveals a clear performance superiority of the SiLU network.}
		\label{fig17}
	\end{figure} 
	Based on some experimental results (Figs.~\ref{fig13} and~\ref{fig14}), we observe good accuracy in the approximation by neural networks with SiLU activation functions via the step functions. The theoretical construction in proof of theorem~\ref{theo8} is experimentally validated with $B = 1$, $n = 5$, $M = 3$ and $2000$ training points, confirming the practical efficiency of SiLU networks for Sobolev function approximation (Fig.~\ref{fig16}).
	
	\paragraph{Performance Analysis}
	
	Experimental implementation (Figs.~\ref{fig17} and~\ref{fig18}) of the architectures in Theorems~\ref{theo1} and~\ref{theo3} for approximating \(x^2\) and \(x y\) highlights the crucial role of parameter initialization. 
	SiLU networks initialized near the analytical solution converge via backpropagation to values close to the theoretical optimum, whereas ReLU networks, whith same architecture, often become trapped in suboptimal minima, yielding less accurate approximations.	
	Our work shows that by allowing weights to scale with precision requirements, one can achieve constant-depth networks for Sobolev approximation. While the total bit complexity remains $\mathcal{O}(\varepsilon^{-d/n}\ln (1/\varepsilon))$ as in Yarotsky's work, the architectural depth is reduced from logarithmic to constant.

	\begin{figure}[H]
		\centering
		\includegraphics[width=\linewidth, height=8cm, keepaspectratio]{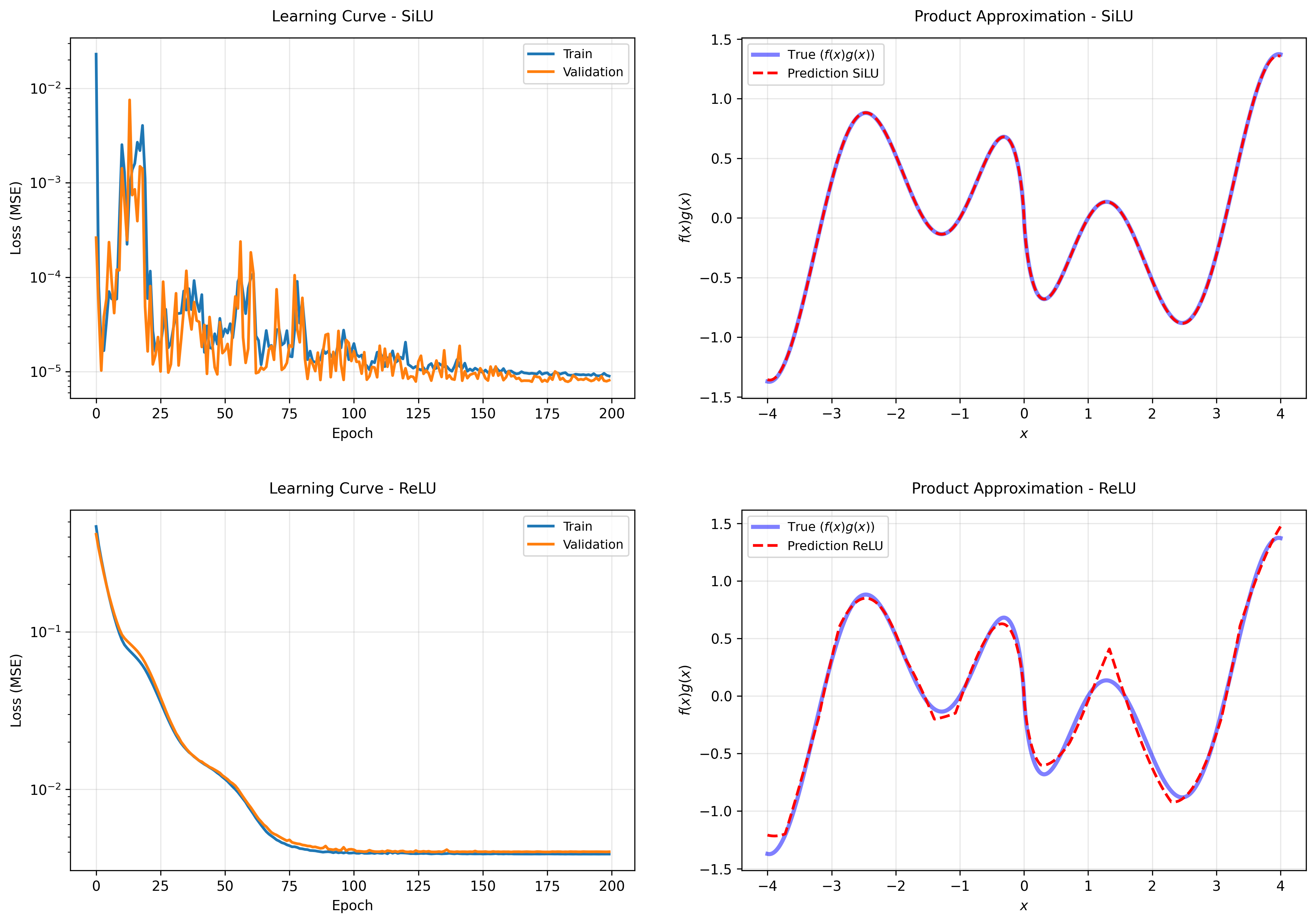}
		\caption{The same observations hold for the network approximating the product $\cos x\cdot \sin x\log(x^2)$ using the architecture from Theorem~\ref{theo2}.} 
		\label{fig18}
	\end{figure} 
	\section{Conclusion}
	This work develops constructive approximation methods for SiLU networks, showing that with precision-tunable activation parameters, one can build constant-depth networks for Sobolev functions using $\mathcal{O}(\varepsilon^{-d/n})$ parameters. Our explicit constructions exhibit exponential convergence for basic operations and reveal parameter-efficient structures for polynomial approximation. Compared to logarithmic-depth ReLU networks, our approach trades architectural depth for increased weight precision, offering an alternative efficiency trade-off. These results advance the understanding of parametric activation networks and open up applicative perspectives in scientific computing where constant-depth architectures may prove advantageous.
	
	\newpage
	
	\section*{Acknowledgments}
	This work was supported by the Agence Française de Développement (AFD) through the Agence Nationale
	de la Recherche (ANR-21-PEA2-0007), under the Partenariats with "l’Enseignement supérieur Africain"
	(PEA) program.

	
	\bibliographystyle{plainnat}  
	
	\bibliography{./biblio_SiLU}

@article{cybenko1989approximation,
	title={Approximation by superpositions of a sigmoidal function},
	author={Cybenko, George},
	journal={Mathematics of control, signals and systems},
	volume={2},
	number={4},
	pages={303--314},
	year={1989},
	publisher={Springer},
	doi = {10.1007/BF02551274}
}

@article{hornik1991approximation,
	title={Approximation capabilities of multilayer feedforward networks},
	author={Hornik, Kurt},
	journal={Neural networks},
	volume={4},
	number={2},
	pages={251--257},
	year={1991},
	publisher={Elsevier},
	doi = {10.1016/0893-6080(91)90009-T}
}

@article{yarotsky2017error,
	title={Error bounds for approximations with deep ReLU networks},
	author={Yarotsky, Dmitry},
	journal={Neural networks},
	volume={94},
	pages={103--114},
	year={2017},
	publisher={Elsevier},
	doi = {10.1016/j.neunet.2017.07.002}
}

@article{hendrycks2016gaussian,
	title={Gaussian error linear units (gelus)},
	author={Hendrycks, Dan and Gimpel, Kevin},
	journal={arXiv preprint arXiv:1606.08415},
	year={2016},
	doi = {10.48550/arXiv.1606.08415}
}

@article{elfwing2018sigmoid,
	title={Sigmoid-weighted linear units for neural network function approximation in reinforcement learning},
	author={Elfwing, Stefan and Uchibe, Eiji and Doya, Kenji},
	journal={Neural networks},
	volume={107},
	pages={3--11},
	year={2018},
	publisher={Elsevier},
	doi = {10.1016/j.neunet.2017.12.012}
}

@article{petersen2018optimal,
	title={Optimal approximation of piecewise smooth functions using deep ReLU neural networks},
	author={Petersen, Philipp and Voigtlaender, Felix},
	journal={Neural Networks},
	volume={108},
	pages={296--330},
	year={2018},
	publisher={Elsevier},
	doi = {10.1016/j.neunet.2018.08.019}
}

@inproceedings{telgarsky2016benefits,
	title={Benefits of depth in neural networks},
	author={Telgarsky, Matus},
	booktitle={Conference on learning theory},
	pages={1517--1539},
	year={2016},
	organization={PMLR}
}

@article{lu2017expressive,
	title={The expressive power of neural networks: A view from the width},
	author={Lu, Zhou and Pu, Hongming and Wang, Feicheng and Hu, Zhiqiang and Wang, Liwei},
	journal={Advances in neural information processing systems},
	volume={30},
	year={2017}
}

@inproceedings{cheang2001penalized,
	title={Penalized least squares, model selection, convex hull classes and neural nets.},
	author={Cheang, Gerald HL and Barron, Andrew R and Verleysen, M},
	booktitle={ESANN},
	pages={371--376},
	year={2001}
}

@article{arora2016understanding,
	title={Understanding deep neural networks with rectified linear units},
	author={Arora, Raman and Basu, Amitabh and Mianjy, Poorya and Mukherjee, Anirbit},
	journal={arXiv preprint arXiv:1611.01491},
	year={2016},
	doi = {10.48550/arXiv.1611.01491}
}

@article{daniely2020neural,
	title={Neural networks learning and memorization with (almost) no over-parameterization},
	author={Daniely, Amit},
	journal={Advances in Neural Information Processing Systems},
	volume={33},
	pages={9007--9016},
	year={2020},
	doi ={10.48550/arXiv.1911.09873}
}

@article{ramachandran2017searching,
	title={Searching for activation functions},
	author={Ramachandran, Prajit and Zoph, Barret and Le, Quoc V},
	journal={arXiv preprint arXiv:1710.05941},
	year={2017},
	doi = {10.48550/arXiv.1710.05941
	}
}

@unpublished{ayena:hal-05178445,
	TITLE = {{A Generalized Variational Meta-Model for Air Pollution Dispersion}},
	AUTHOR = {Ayena, Koffi O and Iovleff, Serge and d'Almeida, Amah S},
	URL = {https://hal.science/hal-05178445},
	NOTE = {working paper or preprint},
	YEAR = {2025},
	MONTH = Jul,
	HAL_ID = {hal-05178445},
	HAL_VERSION = {v1},
}

@article{lee2023gelu,
	title={Gelu activation function in deep learning: a comprehensive mathematical analysis and performance},
	author={Lee, Minhyeok},
	journal={arXiv preprint arXiv:2305.12073},
	year={2023},
	doi = {10.48550/arXiv.2305.12073}
}

@article{funahashi1989approximate,
	title={On the approximate realization of continuous mappings by neural networks},
	author={Funahashi, Ken-Ichi},
	journal={Neural networks},
	volume={2},
	number={3},
	pages={183--192},
	year={1989},
	publisher={Elsevier},
	doi = {10.1016/0893-6080(89)90003-8}
}

@book{jordan1965calculus,
	title={Calculus of finite differences},
	author={Jord{\'a}n, K{\'a}roly},
	volume={33},
	year={1965},
	publisher={American Mathematical Soc.}
}

	\appendix
	
	\section{An upper bound of $|\sigma^{(n)}(a)|$}
	\label{bound}
	The function $\sigma(x) = (1+e^{-x})^{-1}$ is analytic in $\mathbb{C}$. For any $a \in \mathbb{R}$ and $0 < r < \sqrt{a^2+\pi^2}$, Cauchy's formula gives:
	\[
	|\sigma^{(n)}(a)| \leq \frac{n!}{r^n} \max_{|x-a|=r} |\sigma(x)|.
	\]
	
	The poles of $\sigma$ are at $x_k = i\pi(1+2k)$, $k\in\mathbb{Z}$, hence the distance from $a$ to the nearest pole is $d(a) = \sqrt{a^2+\pi^2} \geq \pi$. 
	Choosing $r = \pi/2$ (which satisfies $r < d(a)$ for all $a \in \mathbb{R}$), we have:
	\[
	\delta(a) := \min_{|x-a|=\pi/2} |1+e^{-x}| > 0,
	\]
	and consequently
	\[
	\max_{|x-a|=\pi/2} |\sigma(x)| \leq \frac{1}{\delta(a)}.
	\]
	
	Therefore,
	\[
	|\sigma^{(n)}(a)| \leq \frac{n!}{(\pi/2)^n} \cdot \frac{1}{\delta(a)}
	= n! \cdot \left(\frac{2}{\pi}\right)^n \cdot C(a),
	\]
	where $C(a) = 1/\delta(a)$ is independent of $n$. In particular, there exists a constant $C'_a$ such that
	\[
	\left|\frac{\sigma^{(n)}(a)}{n!}\right| \leq C'_a \left(\frac{2}{\pi}\right)^n.
	\]
	
	\section{Finite Difference Operator $\Delta$}
	\label{diff}
	The finite difference operator $\Delta$ is define in \cite{jordan1965calculus}.
	
	Let \( \Delta f(x) = f(x+1) - f(x) \).
	Noting that \( \Delta = E - I \) where \( E f(x) = f(x+1) \) and $I f(x) = f(x)$, the binomial theorem of Newton allows us to write
	\[
	\Delta^m = (E - I)^m = \sum_{j=0}^m \binom{m}{j} E^j (-I)^{m-j}
	\]
	Since \( E^j f(0) = f(j) \),
	we finally obtain the formula:
	\[
	\Delta^m f(0) = \sum_{j=0}^m (-1)^{m-j} \binom{m}{j} f(j)
	\]
	Let us recall a useful property of finite differences of falling factorials:
	\[
	x^{\underline{k}} = x(x-1)\dots(x-k+1)
	\]
	of degree \( k \), leading coefficient 1.
	
	\begin{lem}[Differences of falling factorials]
		For $k\ge m$,
		\[
		\Delta x^{\underline{k}} = k\,x^{\underline{k-1}}, \qquad
		\Delta^{\,m} x^{\underline{k}} = k^{\underline{m}}\,x^{\underline{k-m}},
		\]
		where $k^{\underline{m}} = k(k-1)\dots(k-m+1)$ is the falling factorial of $k$.
	\end{lem}
	
	\begin{proof}
		Using $ \Delta x^{\underline{k}} = (x+1)^{\underline{k}} - x^{\underline{k}}$ and factoring
		$x(x-1)\dots(x-k+2)$ gives
		\[
		\Delta x^{\underline{k}} = \big[(x+1)-(x-k+1)\big]\,x(x-1)\dots(x-k+2)=k\,x^{\underline{k-1}}.
		\]
		The formula for $\Delta^{\,m}$ follows by induction on $m$: assuming it holds for $m$,
		\[
		\Delta^{\,m+1}x^{\underline{k}} 
		= \Delta\bigl[k^{\underline{m}}\,x^{\underline{k-m}}\bigr] 
		= k^{\underline{m}}\,(k-m)\,x^{\underline{k-m-1}} 
		= k^{\underline{m+1}}\,x^{\underline{k-(m+1)}}. \qedhere
		\]
	\end{proof}
	\newpage	
	
\end{document}